\documentclass[11pt]{article}

\usepackage[final]{acl}

\usepackage{times}
\usepackage{latexsym}
\usepackage[T1]{fontenc}
\usepackage[utf8]{inputenc}
\usepackage{microtype}
\usepackage{inconsolata}
\usepackage{graphicx}
\usepackage{amsmath}
\usepackage{amssymb}
\usepackage{booktabs}
\usepackage{multirow}
\usepackage{tabularx}
\usepackage{hyperref}
\usepackage[table]{xcolor}
\usepackage{subcaption}
\usepackage{enumitem}
\usepackage{tikz}
\usetikzlibrary{arrows.meta, positioning, calc, decorations.pathreplacing}
\usepackage{amsthm}
\usepackage{pifont}
\usepackage{adjustbox}
\usepackage{placeins}

\newtheorem{theorem}{Theorem}
\newtheorem{lemma}{Lemma}

\theoremstyle{remark}
\newtheorem{remark}{Remark}

\theoremstyle{definition}
\newtheorem{assumption}{Assumption}

\DeclareFontShape{T1}{ptm}{m}{scit}{<-> ssub * ptm/m/it}{}
\DeclareFontShape{T1}{ptm}{b}{scit}{<-> ssub * ptm/b/it}{}

\newcommand{\rmsnorm}{\mathrm{RMSNorm}}
\newcommand{\qknorm}{\textsc{QK-RMSNorm}}
\newcommand{\opnorm}[1]{\|#1\|_{\mathrm{op}}}
\newcommand{\frob}[1]{\|#1\|_F}
\title{Text Capability Loss in Vision-Language Adaptation:\\ An Attention-Sink Diagnosis}

\makeatletter
\ifacl@finalcopy
\author{
  Minsik Choi\textsuperscript{*} \\
  Korea University \\
  \texttt{mszzang2002@korea.ac.kr} \\
  \And
  Geewook Kim\textsuperscript{*} \\
  NAVER Cloud AI \\
  KAIST AI \\
  \texttt{gwkim.rsrch@gmail.com} \\
  \And
  Young Geun Kim\textsuperscript{\dag} \\
  Korea University \\
  \texttt{younggeun\_kim@korea.ac.kr} \\
}
\else
\author{Anonymous ACL submission}
\fi
\makeatother

\begin{document}
\maketitle

\makeatletter
\ifacl@finalcopy
\renewcommand{\thefootnote}{\fnsymbol{footnote}}
\footnotetext[1]{Minsik Choi and Geewook Kim contributed equally to this work and share first authorship.}
\footnotetext[2]{Corresponding author.}
\renewcommand{\thefootnote}{\arabic{footnote}}
\setcounter{footnote}{0}
\fi
\makeatother

\begin{abstract}
Fine-tuning a pretrained LLM into a vision-language model (VLM) can erode the backbone's text capability, with the damage concentrated on tasks that require following exact output rules, such as instruction following, chain-of-thought reasoning graded on a strictly parsed final answer, and similar evaluations with strict graders.
We trace this gap to \emph{attention-sink corruption}: VL fine-tuning perturbs the early sink position that anchors a large fraction of attention probability, and how well the base LLM preserves its sink tracks how much of the affected capability survives adaptation.
Building on this view, we introduce \emph{Sink Strength}, a single scalar computed on the base LLM in a few seconds on a single GPU that predicts post-VL degradation without any VL training.
It consistently tracks relative degradation across the six VLM–LLM pairs and multiple format-sensitive tasks.
Complementing this diagnostic, we find that post-pretraining
$\qknorm$ injection fails to reproduce the protection of native
$\qknorm$, while several off-the-shelf weight-merging settings
fail to recover the lost capability after VL training.
These negative results underscore the value of screening
backbones with Sink Strength before VL training and narrow
the intervention space toward head-selective training-time protection.
\end{abstract}

\section{Introduction}
\label{sec:intro}

Vision-language models (VLMs) are commonly built by attaching a vision encoder and projector to a pretrained large language model (LLM) and jointly fine-tuning on multimodal data.
We find that this adaptation can erode the LLM's text capability, with the damage concentrated on tasks we call \emph{format-sensitive}: those that require following exact output rules, such as instruction following or chain-of-thought reasoning whose final answer is parsed under a strict rule.
Across the released VLM--LLM pairs we study, the text-capability
gap reaches double digits on multiple such tasks.
A text-only further-training counterpart shows substantially smaller degradation, while a separately matched VL-versus-text training trajectory isolates the modality difference (Sections~\ref{sec:problem_setup} and~\ref{sec:modality_control}).
Existing remedies each help in part but at a cost, and none explains \emph{why} the loss happens.
The existing remedies include re-blending text-only data into the multimodal supervised fine-tuning (SFT) mix~\citep{lin2024vila,tong2024cambrian,tu2025mlan}, freezing the LLM~\citep{zhu2024minigpt}, low-rank adapters (LoRA)~\citep{liu2024improved,hu2022lora}, and post-hoc neuron re-merging~\citep{yu2025locate}.

We trace this loss to \emph{attention-sink corruption}.
A modern LLM concentrates a large fraction of its attention onto a small number of early positions~\citep{xiao2024efficient,barbero2025llms}, and the hidden states at
those positions are dominated by a few fixed feature dimensions carrying massive
activations~\citep{sun2024massive}.
This sink absorbs diffuse attention mass that would otherwise spread across non-informative positions~\citep{xiao2024efficient}.
We hypothesize that this stabilizes per-position attention to the surface tokens that format-sensitive evaluators grade, and the rest of the paper tests that hypothesis.
We show that VL fine-tuning perturbs the query and key projections, $W_q$ and $W_k$,
that read those dimensions; because the induced logit shift scales with input
magnitude, the perturbation is amplified precisely at the sink, collapsing the per-head sink (Figure~\ref{fig:teaser}).
How well the base LLM's sink survives this perturbation predicts how much post-VL damage we see on format-sensitive tasks.
Base LLMs whose query and key projections are less sensitive to
input magnitude exhibit greater sink preservation under VL adaptation (Section~\ref{sec:method}).
Architectures bound this sensitivity by normalizing the projections, but the granularity matters.
Applied per head, the normalization removes the raw-magnitude dependence outright.
Applied across a whole layer, it instead ties each head's scale to that head's share of the layer's energy, so a weak head is compressed rather than normalized.
Such a backbone reaches VL training with its per-head sink already weak.
Predicting post-VL damage therefore requires accounting for corruption inherited from pre-training as well as corruption induced by VL fine-tuning.
Preserved-sink and collapsed-sink backbones separate cleanly on instruction following: every backbone that preserves its sink stays within $5.4$ pt of its reference LLM, while every backbone whose sink collapses loses at least $7.9$ pt.

\begin{figure}[t]
\centering
\includegraphics[width=\columnwidth]{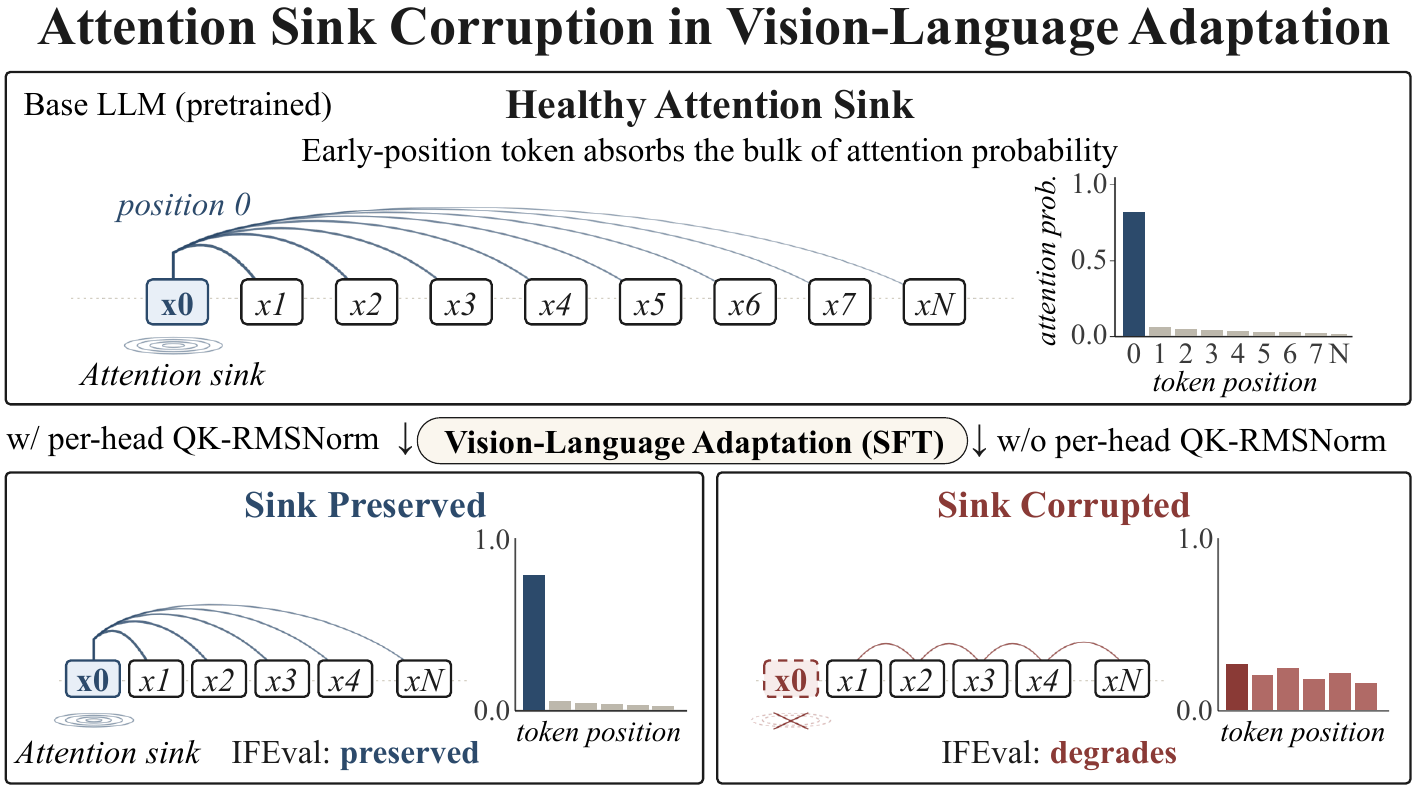}
\caption{Attention-sink corruption in VL fine-tuning: the base LLM concentrates attention on an early sink position; whether this concentration survives VL fine-tuning varies sharply across language backbones.}
\label{fig:teaser}
\end{figure}

Building on this account, we introduce \emph{Sink Strength} $S$ (Section~\ref{sec:predictor}), a single scalar computed entirely on the base LLM using a small number of inference-only forward passes and no VL training.
$S$ quantifies the head-level sink concentration, ranks backbones by their post-VL damage (Spearman $\rho = 0.97$), and the ranking remains strongly correlated across multiple format-sensitive tasks.
Because $S$ is measured before VL training, it exposes base-side sink weakness independently of the subsequent VL update; Section~\ref{sec:method} then combines this quantity with the post-VL gap perturbation to separate base-side vulnerability from update-induced corruption.

Given a high-risk backbone, can simple interventions prevent or recover the damage?
We test two complementary negative controls (Section~\ref{sec:controls}): post-pretraining injection of the $\qknorm$ module into a matched base LLM followed by VL training, and several post-VL weight-merging methods spanning linear, spectral, and sparsified families~\citep{ilharco2022editing,yadav2023ties,yu2024language,gargiulo2025task,lee2025star,wortsman2022model}.
The former fails to reproduce the protection associated with native $\qknorm$, while the latter fails to recover the lost capability under the settings we test, narrowing the intervention space toward head-selective training-time protection.

Our contributions are:
(i)~\textbf{Sink Strength $S$}, a base-LLM scalar predicting post-VL damage on format-sensitive tasks from inference-only forward passes (rank correlations remain strong across our six VLM--LLM pairs and four tasks, $\rho \ge 0.88$), providing a pre-VL screen;
(ii) a \textbf{sink-anchored mechanism} (Theorem~\ref{thm:perturb})
showing how per-head $\qknorm$ on the query and key projections
removes explicit raw-input-magnitude sensitivity from the sink-gap
perturbation bound, supported by margin calibration across the five headline pairs and two layerwise-$\qknorm$ controls;
(iii)~\textbf{confound and negative controls}: text-only versus VL controls supporting modality-specific degradation, together with two complementary negative controls---post-pretraining $\qknorm$ injection and post-VL weight-space merging.

\begin{figure*}[t!]
\centering
\includegraphics[width=\textwidth]{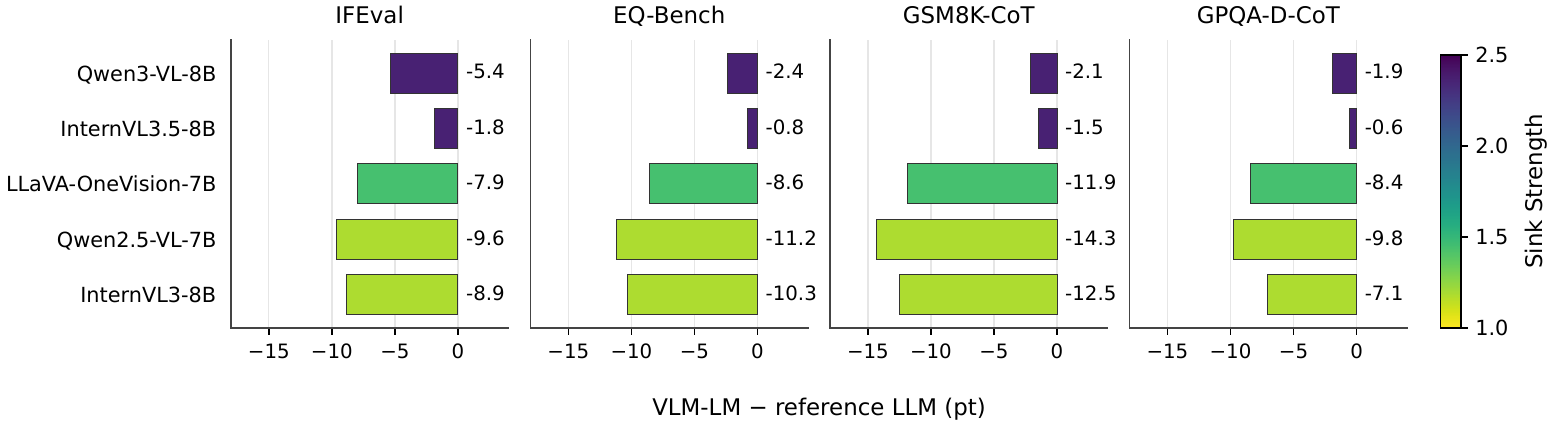}
\caption{Text-capability gap across our headline VLM--LLM pairs on instruction-following and three further format-sensitive tasks. Color: Sink Strength $S$ measured on the reference LLM (Section~\ref{sec:predictor}).}
\label{fig:damage_bars}
\end{figure*}
\section{Background \& Related Work}
\label{sec:related}

\paragraph{VL fine-tuning and text-side erosion.}
Standard VLM recipes~\citep{liu2023visual,liu2024improved} update the LLM decoder weights during adaptation, typically over multiple stages of alignment and instruction tuning.
A recurrent observation across this paradigm is that the decoder loses some of the underlying LLM's text-side capability~\citep{zhang2024wings,zhai2024investigating,lee2025does,ratzlaff2025training,srivastava2024improving}, with instruction-following and other format-sensitive behaviors among the hardest hit.
Existing remedies are data-side (text-only re-blending~\citep{lin2024vila,tong2024cambrian,tu2025mlan}), training-side (frozen LLM~\citep{zhu2024minigpt} or low-rank adapter~\citep{hu2022lora,liu2024improved}), or post-hoc (neuron-level re-merging~\citep{yu2025locate}), each closing part of the gap.
Recent work has also explored weight merging within multimodal instruction-tuning pipelines \citep{choi2026decentralized}.
Our primary diagnostic analysis leaves the released VLM training recipes untouched and instead asks which base-LLM property predicts how much text capability survives adaptation.

\paragraph{Attention sinks and mechanistic position.}
The attention-sink phenomenon~\citep{xiao2024efficient} has been traced to massive activations~\citep{sun2024massive}, given an over-mixing rationale~\citep{barbero2025llms}, linked to embedding-level bookkeeping~\citep{zhang2025attention}, and its emergence conditions characterized~\citep{gu2025attention}, with the vision analogue motivating register tokens~\citep{darcet2024vision}.
While these characterize the sink of a fixed network at inference time, we study what happens to it \emph{during} fine-tuning.
VL fine-tuning is known to shift attention patterns within the adapted decoder~\citep{zhang2024wings}, and we operationalize this shift as corruption of the early-position sink that anchors text-side stability.
$\qknorm$ replaces the L2 normalization of the original query--key normalization~\citep{henry2020query} with RMSNorm, motivated by entropy-collapse stability arguments~\citep{zhai2023stabilizing}.
$\qknorm$ adoption varies across base LLMs: present in Qwen3~\citep{yang2025qwen3}, absent in Qwen2.5~\citep{qwen2025qwen25technicalreport} and Llama-3~\citep{grattafiori2024llama}.
This architectural axis is examined throughout our analysis.
Our bound and saturation lemma are in the spirit of small-circuit attention analyses~\citep{olsson2022context,geva2021transformer} and fine-tuning circuit-stability work~\citep{wang2025towards}.

\section{Sink Strength: A Pre-VL Diagnostic}
\label{sec:predictor}

\subsection{Problem Setup}
\label{sec:problem_setup}

An LM-only further-training counterpart shows substantially smaller
text-capability degradation than the corresponding VL endpoint across
the four format-sensitive tasks; Section~\ref{sec:controls} provides a separately
matched VL-versus-text training trajectory from a shared base LLM.

\paragraph{Setup.}
The five headline pairs come from three model families:
Qwen, InternVL, and LLaVA-OneVision.
Full per-pair metadata is in Table~\ref{tab:pair_metadata}.\footnote{Hugging Face URLs and paper references for every model and evaluation dataset are collected in Appendix~\ref{app:datasets}.}
All released VLMs in the headline panel update the language backbone rather than relying on LoRA-only adaptation, removing the adapter-merge confound.
For each released VLM, we extract its language backbone as a standalone causal LM (the VLM-LM hereafter; Appendix~\ref{sec:extraction} and score it on a nine-task text suite (Appendix~\ref{sec:protocol}) against the corresponding text-only reference LLM, a comparison point rather than an assumed immediate predecessor of multimodal training.

\begin{table}[t]
\centering
\setlength{\tabcolsep}{3pt}
\fontsize{8}{10}\selectfont
\begin{adjustbox}{max width=\columnwidth}
\begin{tabular}{@{}l l l l@{}}
\toprule
Pair & Reference LLM & Post-training & QK \\
\midrule
Qwen3-VL     & Qwen3-8B            & SFT+RL       & \textbf{per-head}  \\
InternVL3.5  & Qwen3-8B            & SFT+RL (CascadeRL)    & \textbf{per-head}  \\
\addlinespace[1pt]
Qwen2.5-VL   & Qwen2.5-7B-Instruct & SFT+DPO    & none      \\
InternVL3    & Qwen2.5-7B-Instruct & SFT+MPO    & none      \\
LLaVA-OV     & Qwen2-7B-Instruct   & SFT      & none      \\
\midrule
\multicolumn{4}{@{}l}{\emph{Non-headline control (Appendix~\ref{app:molmo})}} \\
Molmo2-O     & Olmo-3-7B-Instruct  & SFT (PixMo) & \emph{layerwise} \\
\bottomrule
\end{tabular}
\end{adjustbox}
\caption{Five headline VLM--LLM pairs and the Molmo2-O control. "QK" reports the per-head / layerwise / none $\qknorm$ configuration. Per-head $\qknorm$ normalizes each attention head by its own RMS, while the layerwise variant applies a single RMS across all concatenated heads.}
\label{tab:pair_metadata}
\end{table}

\paragraph{Format-sensitive tasks.}
We call IFEval, EQ-Bench, GSM8K-CoT, and GPQA-Diamond-CoT \emph{format-sensitive} because each grades a strict, parseable surface gate: per-prompt format compliance for IFEval~\citep{zhou2023instruction} (our headline metric), score-block parsing for EQ-Bench, and CoT answer-token extraction for GSM8K-CoT and GPQA-Diamond-CoT.
Because these gates are parsed rather than scored semantically, small attention-pattern drift can disproportionately flip the outcome.
Concretely, the regressing prompts we inspect often fail parseable rules: length constraints (a $1200$-word essay truncated), exact-phrase requirements ("My answer is maybe" reduced to "My answer"), single-language constraints, and format markers.
A per-category breakdown across pairs is in Appendix~\ref{app:failure_categories}.

\paragraph{Per-pair text-capability gap.}
Figure~\ref{fig:damage_bars} (IFEval row) shows a clear VLM--LLM text-capability gap across our five headline pairs: two pairs show gaps of at most 5.4 pt (InternVL3.5 at $-1.8$, Qwen3-VL at $-5.4$), while three show gaps of at least 7.9 pt (LLaVA-OV $-7.9$, InternVL3 $-8.9$, Qwen2.5-VL $-9.6$).
The non-headline pair (Molmo2-O on Olmo-3) shows a $-18.7$-pt gap and is analyzed separately as a control (Section~\ref{sec:method}).

\paragraph{Per-prompt regression.}
On the $541$ IFEval validation prompts each backbone evaluates, we tag every prompt as \emph{both\_pass} (LLM and VLM-LM both score $\ge\!0.5$), \emph{regression} (LLM pass, VLM-LM fail), \emph{recovered} (LLM fail, VLM-LM pass), or \emph{both\_fail} (Table~\ref{tab:regression}).
The net regression rate reproduces the IFEval delta pair-by-pair.
The regressions are spread across $43$--$126$ distinct prompts per pair, so the behavioral split reflects broad instruction-following degradation rather than a few outlier prompts.

\begin{table}[t]
\centering
\setlength{\tabcolsep}{3pt}
\fontsize{8}{10}\selectfont
\begin{adjustbox}{max width=\columnwidth}
\begin{tabular}{@{}l l r r r r r@{}}
\toprule
Pair & QK & Both\,$\checkmark$ & Lost & Gained & Both\,$\times$ & Net (\%) \\
\midrule
Qwen3-VL     & per-head  & 394 & 56  & 27 & 64  & $\mathbf{5.4}$ \\
InternVL3.5  & per-head  & 407 & 43  & 33 & 58  & $\mathbf{1.8}$ \\
\addlinespace[1pt]
Qwen2.5-VL   & none      & 299 & $\mathbf{91}$  & 39 & 112 & $9.6$ \\
InternVL3    & none      & 294 & $\mathbf{96}$  & 48 & 103 & $8.9$ \\
LLaVA-OV     & none      & 195 & $\mathbf{83}$  & 40 & 223 & $7.9$ \\
\midrule
Molmo2-O     & layerwise & 317 & $\mathbf{126}$ & 25 & 73  & $\mathbf{18.7}$ \\
\bottomrule
\end{tabular}
\end{adjustbox}
\caption{Per-prompt regression breakdown on the $541$-prompt IFEval set, partitioning prompts by reference-LLM and VLM-LM pass/fail. Net $=$ $(\text{Lost}-\text{Gained})/541$.}
\label{tab:regression}
\end{table}

\paragraph{Modality attribution of the loss.}
A natural question is whether the observed gap reflects further optimization in general rather than VL training specifically.
We first compare two Qwen2.5-family endpoints against the same Qwen2.5-7B-Instruct text-only reference: the text-only Qwen2.5-7B-Instruct-1M variant and Qwen2.5-VL-7B-Instruct.
Under the same instruct protocol (0-shot with the chat template applied, Appendix~\ref{sec:protocol}), the text-only variant matches the reference on IFEval prompt-strict while the VL endpoint shows a $9.6$-pt gap; across the four format-sensitive tasks the text-only variant does not regress on IFEval or GSM8K and loses at most $0.42\times$ the VL gap on EQ-Bench and GPQA (Table~\ref{tab:lm_sft_control}; details in Appendix~\ref{app:lm_sft}).
The same comparison also indicates why the VL endpoint perturbs the sink more strongly.
Across every layer of the late-layer window, its displacement
from the reference is larger in operator norm than that of the
text-only endpoint, with a median ratio of $4.44\times$ on the
query projection and $3.65\times$ on the key projection.
The sink-to-other input contrast is also $\approx1.12\times$
larger once vision tokens are present.
This endpoint contrast supports modality as an important source of the gap rather than further optimization alone, with a separately matched training-time control provided in Section~\ref{sec:controls}.
This frames the predictor question: which candidate LLM backbones are most vulnerable to text-side degradation under VL adaptation?

\begin{table}[t]
\centering
\setlength{\tabcolsep}{4pt}
\fontsize{8}{10}\selectfont
\begin{tabular}{@{}l r r@{}}
\toprule
Bench & Text-only $\Delta$ & VL $\Delta$ \\
\midrule
IFEval prompt-strict     & $\phantom{-}0.00$ & $\mathbf{-9.6}$ \\
GSM8K-CoT (flex)         & $+5.46$           & $\mathbf{-14.3}$ \\
GPQA-Diamond-CoT (flex)  & $-4.04$           & $\mathbf{-9.8}$ \\
EQ-Bench v2.1            & $-2.24$           & $\mathbf{-11.2}$ \\
\bottomrule
\end{tabular}
\caption{Endpoint modality comparison against the same Qwen2.5-7B-Instruct text-only reference: the text-only endpoint is Qwen2.5-7B-Instruct-1M, while the VL endpoint is Qwen2.5-VL-7B-Instruct.
}
\label{tab:lm_sft_control}
\end{table}

\subsection{Sink Strength}
\label{sec:sink_strength}

Sink Strength $S$, measured on each text-only reference LLM, captures the per-pair capability-gap split, whose clustering further aligns with per-head $\qknorm$ in the reference LLM.

\paragraph{Definition.}
We capture per-head sink concentration on the base LLM as a stand-alone metric, the \emph{Sink Strength}:
\[
S \;:=\; \mathrm{median}_{(\ell, h, q)} \;\log\frac{a_{p_{\mathrm{sink}}}^{(\ell,h,q)}}{1 - a_{p_{\mathrm{sink}}}^{(\ell,h,q)}},
\]
where $a_{p_{\mathrm{sink}}}^{(\ell, h, q)}$ is the attention probability at the per-head argmax key position $p_{\mathrm{sink}}$ at layer $\ell$, head $h$, query position $q$, and the median is taken over the last $\sim\!10$ layers (design ablations in Appendix~\ref{app:sink_design_ablations}).
$S$ is computed entirely on the base LLM (no VLM forward pass, no VL training) with $15$ inference-only forward passes on calibration prompts disjoint from the IFEval test set.
For a 7B backbone, this takes $\sim\!8$ seconds on a single A6000 GPU.
Intuitively, $S$ quantifies how concentrated the base LLM's attention is on its sink position: high $S$ means a strong, well-formed sink, while $S \le 0$ means the sink does not dominate the aggregate non-sink attention mass.
Section~\ref{sec:method} formalizes why this scalar captures vulnerability to VL-induced sink corruption via a sink-anchored perturbation bound.

\paragraph{Analysis.}
Color in Figure~\ref{fig:damage_bars} indicates per-pair $S$ measured on the reference LLM, and the same backbones that lose more on IFEval have smaller $S$.
$S$ separates the six pairs into non-overlapping ranges along both axes: the two lower-loss pairs share $S=+2.36$ (sharing the Qwen3-8B reference), the three higher-loss pairs fall in $S \in [+1.18, +1.44]$, and the Molmo2-O boundary case sits at $S \!=\! +0.24$ (Figure~\ref{fig:predictor}).
The ordering induced by $S$ strongly tracks the IFEval-gap ordering across the six pairs.

\begin{figure}[t]
\centering
\includegraphics[width=0.95\columnwidth]{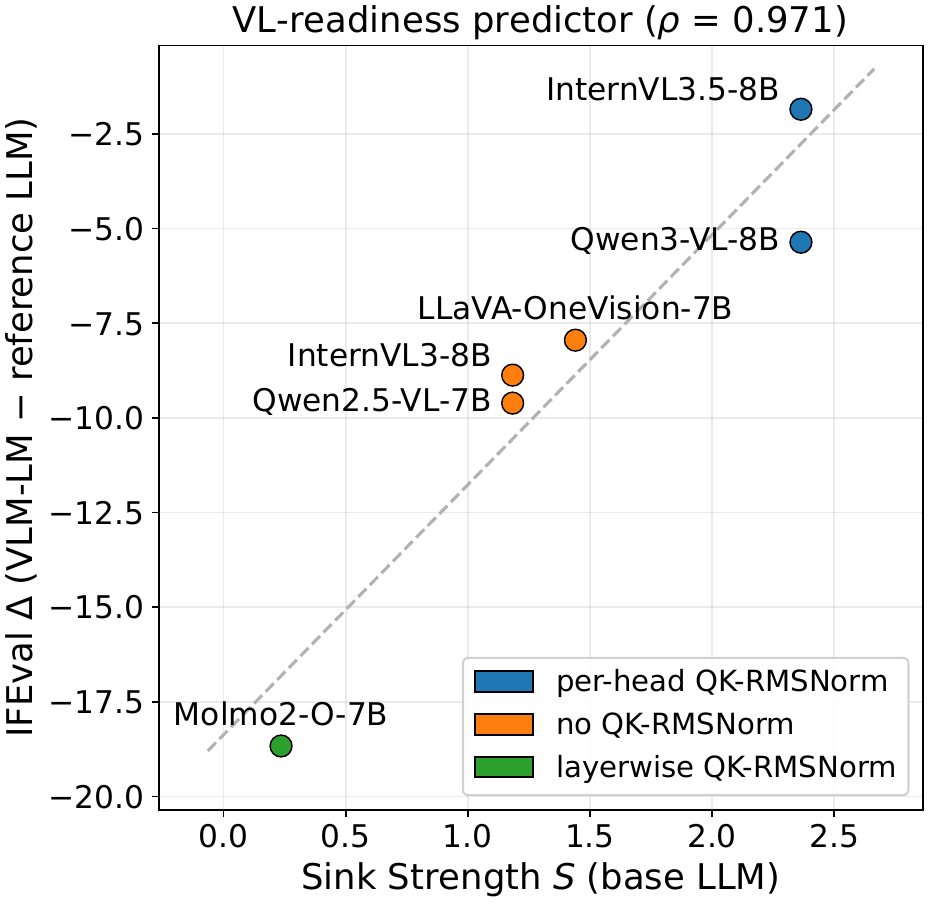}
\caption{Sink Strength $S$ on the reference LLM versus the VLM--LLM IFEval gap. Color denotes the $\qknorm$ configuration. The dashed line is a global linear fit for visualization; reported prediction errors use leave-one-pair-out fitting.}
\label{fig:predictor}
\end{figure}

\paragraph{From sink collapse to format-sensitive loss.}
Low $S$ in a candidate base LLM indicates a sink that fails to dominate non-informative positions: a sink collapse waiting to surface under VL fine-tuning.
Format-sensitive tasks such as instruction following, score-block extraction, and CoT answer parsing grade strict, parseable rules over surface tokens that require stable per-position attention; the early-position sink anchors this stability~\citep{xiao2024efficient,barbero2025llms}, so sink collapse breaks surface-token tracking and flips format-checker outcomes.
Knowledge tasks, which rely on facts encoded in MLP weights~\citep{geva2021transformer}, do not depend on the same attention-pattern stability and are correspondingly preserved (Appendix~\ref{app:knowledge_coverage}).
The same holds on a broader capability panel spanning coding, long-context retrieval, and advanced reasoning, where the mean degradation remains substantially smaller than on the format-sensitive tasks (Appendix~\ref{app:extended_bench}).

\paragraph{Architectural correlate.}
The $S$ clusters align with per-head $\qknorm$ presence in the reference LLM, with finer resolution: per-head $\qknorm$ pairs occupy the high-$S$ range, no-$\qknorm$ pairs the mid-$S$ range, and Molmo2-O is a boundary case where $\qknorm$ is present but applied layerwise.
A coarse "has any $\qknorm$ module" classifier would flag Molmo2-O as safe; $S$ correctly identifies it as the highest-risk backbone among the six pairs, matching the largest measured drop in that set ($-18.7$ pt on IFEval).
Because $S$ is measured on the base LLM alone, this distinction is available \emph{before} any VL training is committed, framing per-head sink preservation as a \emph{pre-training architectural-design concern}.
Section~\ref{sec:method_bound} gives the mechanism: per-head normalization removes the raw-magnitude dependence from the perturbation bound, while layerwise normalization leaves each head's scale tied to its share of the layer's energy (Remark~\ref{rem:layerwise}).
Section~\ref{sec:method_perhead} confirms this empirically on Molmo2-O.
Within pairs sharing the same reference LLM, and hence the same $S$, the drop still varies by up to $3.5$ pt (Qwen3-VL vs.\ InternVL3.5), which $S$ cannot resolve and we attribute to VL recipe variation.

\paragraph{Empirical validation.}
We use $S$ in two modes: rank ordering across backbones requires no outcome calibration, while magnitude prediction adds a single 1-D linear fit calibrated on known post-VL outcomes.
Across the six pairs of Table~\ref{tab:pair_metadata}, $S$ strongly tracks the VLM--LLM IFEval gap (Spearman $\rho = 0.97$); under leave-one-pair-out validation, the average held-out prediction error is $2.54$ pt on a $17$-pt outcome range (best $-1.8$ pt to worst $-18.7$ pt), against $4.40$ pt for a constant-mean baseline (per-pair detail in Appendix~\ref{app:predictor_detail}).
With six datapoints this is a sanity check on the linear relation rather than prospective evidence on unseen backbones.
The one large headline residual is Qwen3-VL ($4.4$ pt held-out): it shares $S = 2.36$ with InternVL3.5 but the two differ by 3.5 pt in the observed gap, and no single-feature fit can separate backbones with identical $S$, which sets the resolution floor of the predictor.

\subsection{Generalization and Practical Use}

\paragraph{Generalization across format-sensitive tasks.}
The same $S$, measured once on the reference LLM, also ranks the backbones on each of the four format-sensitive tasks (Table~\ref{tab:predictor_cross}).
Across six pairs, Spearman $\rho \in [0.88, 0.97]$ on every task.
On the five headline pairs (excluding Molmo2-O), the average held-out error stays under $2$ pt on each task.
$S$ generalizes as a rank predictor across tasks without re-measurement; per-task magnitude prediction additionally requires a per-task linear fit but reuses the same reference-LLM $S$.

\paragraph{Predictor extension.}
We extend the predictor along two complementary axes.
At fixed IFEval, a 17-pair panel spanning dense and MoE architectures, four reference-LLM families plus OLMoE, and 1.5B--32B scale achieves Spearman $\rho=0.88$ (Appendix~\ref{app:predictor_17pair}).
On a 9-pair subset evaluated across all four format-sensitive tasks, the rank correlation remains $\rho\ge0.82$
(Appendix~\ref{app:predictor_extended}).

\begin{table}[t]
\centering
\setlength{\tabcolsep}{3pt}
\fontsize{8}{10}\selectfont
\begin{adjustbox}{max width=\columnwidth}
\begin{tabular}{@{}l c c c c@{}}
\toprule
 & IFEval & EQ-Bench & GSM8K & GPQA-D \\
\midrule
Spearman $\rho$, n$=$6        & $0.971$ & $0.971$ & $0.971$  & $0.883$ \\
Error, 6 pairs (pt)           & $2.54$  & $19.09$ & $7.91$   & $1.96$  \\
Error, 5 headline pairs (pt)  & $1.63$  & $0.95$  & $1.03$   & $1.69$  \\
\bottomrule
\end{tabular}
\end{adjustbox}
\caption{Sink Strength $S$ as a cross-task predictor on the four format-sensitive tasks (held-out 1-D linear regression). The 5-headline row excludes Molmo2-O, the boundary case where the linear extrapolation breaks down (task-specific Molmo collapse, e.g., EQ-Bench $-64.6$); rank correlation holds on all six pairs.}
\label{tab:predictor_cross}
\end{table}

\paragraph{Practical use.}
Given a candidate base LLM, $S$ can be measured with 15 calibration
prompts in seconds on a single GPU and used to rank expected post-VL IFEval damage.
In our panels, every pair at $S\ge2$ stays within 5.4 pt on IFEval,
whereas pairs near $S\approx1.2$ span $-8.9$ to $-17.9$ pt.

\paragraph{From diagnosis to mechanism.}
Section~\ref{sec:method} explains why a quantity measured before VL training predicts what survives it.
There we bound the logit-gap perturbation $B_{\mathrm{gap}}$ that an update of bounded size can inflict (Theorem~\ref{thm:perturb}), and show that the surviving sink mass is governed by the margin $G_{\mathrm{base}} - B_{\mathrm{gap}}$ (Lemma~\ref{lem:saturation}).
Since $S$ is the median of $G_{\mathrm{base}}$, it measures the base-side term of that margin: a high-$S$ backbone enters VL training with more sink lead than a bounded update can close.
The other term, $B_{\mathrm{gap}}$, is recipe-dependent and observable only after adaptation, which is why $S$ alone suffices to rank backbones beforehand while the full margin refines the analysis post hoc.
In our sample the base-side term alone is in fact the stronger signal, with $S$ tracking the IFEval gap better than the full margin ($\rho = 0.97$ vs.\ $0.89$), so we use $S$ as a leading-order, low-cost proxy for it.

\section{Mechanism}
\label{sec:method}

A sink survives as long as its attention logit stays clearly ahead of the logits at the other key positions.
This section makes that lead quantitative and bounds how much of it a fine-tuning update of a given size can close.

\subsection{Setup and Notation}
\label{sec:method_setup}

\paragraph{Setup.}
Fix a single attention head (hidden size $d$, head dim $d_h$; the layer has $H$ heads), a single layer, and a single query token, with base-LLM hidden states $\{\xi_q, \xi_p\}_p$ of the surrounding tokens, where $\xi_q = \gamma_{\mathrm{in}} \odot \rmsnorm(x)$ for input-layernorm scale $\gamma_{\mathrm{in}}$ and $\rmsnorm$ denotes root-mean-square normalization.
The attention logit at key position $p$ is
\begin{align}
l_p &= \frac{Q \cdot K_p}{\sqrt{d_h}}, \nonumber \\
Q &= W_q \xi_q, \qquad K_p = W_k \xi_p .
\end{align}
Under $\qknorm$ each projection is divided by its own RMS before the scale is applied,
\begin{equation}
Q = \gamma_q \odot \frac{Q_{\mathrm{pre}}}{\mathrm{RMS}(Q_{\mathrm{pre}})} ,
\end{equation}
where $Q_{\mathrm{pre}} = W_q \xi_q$, and analogously for $K_p$.
Here $\|\gamma_*\|_\infty := \max_d |\gamma_*[d]|$ denotes the channelwise $\ell_\infty$ norm of a $\qknorm$ scale vector.
Let $l_p' = l_p + \delta_p$ be the post-perturbation logit, and let $p_{\mathrm{sink}}$ be the base LLM's per-head sink position (argmax key on the unperturbed forward pass).
Two quantities anchored at $p_{\mathrm{sink}}$ drive the bound,
\begin{align}
G_{\mathrm{base}} &:= \log \frac{a_{p_{\mathrm{sink}}}}{1 - a_{p_{\mathrm{sink}}}}, \nonumber \\
B_{\mathrm{gap}} &:= \sup_{p \ne p_{\mathrm{sink}}} |\delta_p - \delta_{p_{\mathrm{sink}}}| ,
\end{align}
the aggregate sink gap on the unperturbed logits and the sink-anchored gap perturbation over the fixed key positions of the calibration prompt.

\paragraph{Perturbation model.}
\begin{assumption}[Perturbation model]
\label{ass:perturb}
At the layer of interest, fine-tuning perturbs only the query and key projections, by $\Delta W_*$ with $\opnorm{\Delta W_*} \le \varepsilon$, and leaves the $\qknorm$ scales $\gamma_q, \gamma_k$ and the input-layernorm scale $\gamma_{\mathrm{in}}$ fixed.
We hold the base hidden states $\{\xi_q, \xi_p\}$ at their LLM baseline values, so the bound isolates the effect of the projection drift at that layer.
\end{assumption}
\noindent Both parts have a plain reading.
The operator-norm cap $\varepsilon$ limits how far fine-tuning can move either projection.
Holding the $\gamma$'s fixed is a close approximation for the Qwen3-VL pair, where the relative-norm drift $\|\Delta\gamma\|/\|\gamma\|$ is at most 2.9\% across the 36 q-norm and k-norm layers.

\paragraph{Notation summary.}
We use four sink-related quantities: $G_{\mathrm{base}}$ (per-head log-odds of sink attention, base LLM), $B_{\mathrm{gap}}$ (post-VL gap perturbation, Theorem~\ref{thm:perturb}), the margin $G_{\mathrm{base}} - B_{\mathrm{gap}}$ (controls sink survival, Lemma~\ref{lem:saturation}), and Sink Strength $S$ (Section~\ref{sec:predictor}, median of $G_{\mathrm{base}}$ over $(\ell, h, q)$).

\subsection{Perturbation Bound and Margin Saturation}
\label{sec:method_bound}

\begin{assumption}[Projection-RMS non-degeneracy]
\label{ass:rms}
Let $\mathcal{S}_q$ be the finite set of query-side inputs $\xi_q$ over the calibration prompts at the layer and head of interest, and analogously $\mathcal{S}_k$ for key-side inputs $\xi_p$ over all $p$.
Assume $\|\xi\|_2 > 0$ and $\mathrm{RMS}(W_* \xi) > 0$ for every $\xi$ in the relevant set, and define $m_q := \min_{\xi \in \mathcal{S}_q} \mathrm{RMS}(W_q \xi)/\|\xi\|_2$ and $m_k := \min_{\xi \in \mathcal{S}_k} \mathrm{RMS}(W_k \xi)/\|\xi\|_2$.
Then $m_q, m_k > 0$ in the finite calibration set we measure.
\end{assumption}

\paragraph{Intuition.}
The sink margin is threatened when a small $\Delta W$ produces a large logit shift $\delta_p$, and the three regimes differ in how the input $\xi$ enters.
Without $\qknorm$, $\delta_p \sim \Delta W \cdot \xi$ carries two input factors.
The first is the query magnitude $\|\xi_q\|$, present at any query position whether or not the sink falls there.
The second is the key-side magnitude term, which includes the sink input $\|\xi_{p_{\rm sink}}\|$ and can therefore be amplified when the sink carries a massive-activation channel~\citep{sun2024massive}.
Per-head $\qknorm$ divides $Q$ and $K_p$ by their own RMS, which absorbs $\|\xi\|$, so the bound depends only on the frozen $\|\gamma\|_\infty$ and the projection-RMS lower bounds, removing the explicit dependence on raw input magnitude.
Layerwise $\qknorm$ also absorbs $\|\xi\|$, through a single RMS over $Q_{\mathrm{cat}} = [Q_1; \cdots; Q_H]$, but the scale it leaves behind depends on each head's share of the layer's energy, so a weak head is compressed rather than normalized.

\begin{theorem}[Logit-gap perturbation]
\label{thm:perturb}
Under Assumptions~\ref{ass:perturb} and~\ref{ass:rms}, in the regime where the first-order Taylor expansion of $\tilde q, \tilde k$ near $W_q, W_k$ is valid (Appendix~\ref{app:proof_thmA}), we have, with an absolute constant $c$ and $O(\varepsilon^2)$ constants depending on the fixed weights $W_q, W_k$, scales $\gamma_q, \gamma_k, \gamma_{\mathrm{in}}$, calibration inputs $\{\xi_q, \xi_p\}$, and RMS lower bounds $m_q, m_k$,
\noindent\emph{No $\qknorm$:}
\begin{align}
B_{\mathrm{gap}} \le \,&
c\,\varepsilon\,d_h^{-1/2}\|\xi_q\|\,(\opnorm{W_q}\!+\!\opnorm{W_k}) \nonumber \\
&\cdot \sup_{p\ne p_{\mathrm{sink}}}
\left(
\|\xi_{p_{\mathrm{sink}}}\|+\|\xi_p\|
\right)
+ O(\varepsilon^2).
\label{eq:thm_perturb_no}
\end{align}

\noindent\emph{With $\qknorm$:}
\begin{align}
B_{\mathrm{gap}} \le \,&
c\,\varepsilon\,\|\gamma_q\|_\infty\|\gamma_k\|_\infty\,
(m_q^{-1}\!+\!m_k^{-1}) \nonumber \\
&+ O(\varepsilon^2).
\label{eq:thm_perturb_yes}
\end{align}
\end{theorem}

The two bounds make this precise, with explicit raw-input-magnitude factors surviving only in the no-$\qknorm$ case.

\begin{remark}[Layerwise $\qknorm$ extension; semi-formal]
\label{rem:layerwise}
Here the $H$ heads share a single $\mathrm{RMS}$ denominator over the concatenated $H \cdot d_h$-dimensional $Q_{\mathrm{cat}}$, with the layer's $\gamma_q$ applied afterwards.
Extracting head $h$ from $\tilde Q_{\mathrm{cat}} = \gamma_q \odot Q_{\mathrm{cat}} / \mathrm{RMS}(Q_{\mathrm{cat}})$ gives
\begin{equation}
\|\tilde Q_h\| \;\le\; \|\gamma_q^{(h)}\|_\infty \|Q_h\| \frac{\sqrt{H d_h}}{\|Q_{\mathrm{cat}}\|} ,
\end{equation}
which replaces the head-energy-independent envelope
$\|\gamma_q^{(h)}\|_\infty\sqrt{d_h}$ that per-head normalization provides; the same applies to keys.
Writing $\rho_h := \|Q_h\|/\|Q_{\mathrm{cat}}\|$ for the
head's share of layer energy, the two agree when $\rho_h = 1/\sqrt{H}$, while a head carrying less than its share
is compressed and its attention logits, including the sink's lead,
are attenuated with it.
The margin can therefore fail from the
$G_{\mathrm{base}}$ side rather than through a larger
$B_{\mathrm{gap}}$, and Molmo2-O is the empirical analogue with the weakest per-head $G_{\mathrm{base}}$ among the calibrated pairs of Table~\ref{tab:G_Bgap} (full derivation in Appendix~\ref{app:layerwise_deriv}; empirical control in Appendix~\ref{app:molmo}).
\end{remark}

\begin{lemma}[Saturation transfer]
\label{lem:saturation}
For sink position $p_{\mathrm{sink}}$ with pre-perturbation aggregate gap $G_{\mathrm{base}}$ and gap perturbation $B_{\mathrm{gap}}$,
\begin{equation}
1 - a'_{p_{\mathrm{sink}}} \;\le\;
\sigma\bigl(B_{\mathrm{gap}} - G_{\mathrm{base}}\bigr),
\label{eq:lem_sat}
\end{equation}
where $\sigma(x) = e^x / (1 + e^x)$ is the logistic function.
\end{lemma}

The margin $G_{\mathrm{base}} - B_{\mathrm{gap}}$ controls an upper bound on the non-sink mass through $\sigma(B_{\mathrm{gap}} - G_{\mathrm{base}})$.
Preserving the sink at level $1 - \eta$ requires $G_{\mathrm{base}} - B_{\mathrm{gap}} \ge \log((1-\eta)/\eta)$, and negative margins are consistent with the observed sink collapse.

\subsection{Empirical Sink Behavior}
\label{sec:method_empirical}

We calibrate the margin quantities on the seven pairs, measure the cross-pair sink-mass distribution, and use the layerwise controls to distinguish aggregate from per-head sink.

\paragraph{Calibration across the seven pairs.}
We compute $(G_{\mathrm{base}}, B_{\mathrm{gap}})$ across all five headline pairs and the two layerwise controls on a $15$-prompt calibration set, aggregating over (layer, head, query position) in the last $\sim\!10$ layers where sink behavior is established beyond the first few local-context layers~\citep{xiao2024efficient}.
$G_{\mathrm{base}}$ is computed on the reference LLM at the per-head argmax key $p_{\mathrm{sink}}$, and $B_{\mathrm{gap}}$ is reconstructed from the change in attention-probability ratios between the reference LLM and the VLM-LM at the same $p_{\mathrm{sink}}$.
The empirical $B_{\mathrm{gap}}$ is the realized reference-to-VLM-LM gap change and thus includes hidden-state and other parameter drift; it enters the margin of Lemma~\ref{lem:saturation} but is not a direct numerical evaluation of Theorem~\ref{thm:perturb}'s projection-only bound.
For released pairs, the reference LLM is a comparison point rather than necessarily the immediate pre-VL checkpoint.
Within each Qwen generation, the two pairs share the same reference LLM and therefore the same $G_{\mathrm{base}}$; per-pair values are reported in Table~\ref{tab:G_Bgap}.
The marginal proxy $\mathrm{median}(G)-\mathrm{median}(B)$ is near zero on both per-head $\qknorm$ pairs and strongly negative on all three no-$\qknorm$ pairs.
Both layerwise $\qknorm$ pairs land at the negative end of the sample, with Molmo2-O driven by the weakest base sink in the set.
The ordering tracks the behavioral split of Figure~\ref{fig:damage_bars}.
Calibration sample-size sensitivity is reported in Appendix~\ref{app:calib_sweep}.

\begin{table}[t]
\centering
\setlength{\tabcolsep}{8pt}
\fontsize{8}{10}\selectfont
\begin{tabular}{@{}l l c c c@{}}
\toprule
Pair & QK & $G_{\mathrm{base}}$ & $B_{\mathrm{gap}}$ & $G\!-\!B$ \\
\midrule
Qwen3-VL     & per-head  & $+2.36$ & $+2.06$ & $\mathbf{+0.31}$ \\
InternVL3.5  & per-head  & $+2.36$ & $+2.47$ & $\mathbf{-0.11}$ \\
\addlinespace[1pt]
Qwen2.5-VL   & none      & $+1.18$ & $+5.01$ & $\mathbf{-3.83}$ \\
InternVL3    & none      & $+1.18$ & $+3.40$ & $\mathbf{-2.22}$ \\
LLaVA-OV     & none      & $+1.44$ & $+2.53$ & $\mathbf{-1.09}$ \\
\midrule
Molmo2-O     & layerwise & $+0.24$ & $+3.25$ & $\mathbf{-3.01}$ \\
MolmoE-1B    & layerwise & $+1.34$ & $+4.19$ & $\mathbf{-2.84}$ \\
\bottomrule
\end{tabular}
\caption{Per-pair $G_{\mathrm{base}}$, $B_{\mathrm{gap}}$, and the marginal proxy $G - B := \mathrm{median}(G_{\mathrm{base}}) - \mathrm{median}(B_{\mathrm{gap}})$ of the Lemma~\ref{lem:saturation} margin, on a $15$-prompt calibration set.}\label{tab:G_Bgap}
\end{table}

\paragraph{Cross-pair sink and outcome bucketing.}
We partition each VLM-LM's IFEval prompts by per-prompt strict accuracy (pass at $\ge 0.5$).
For each bucket we measure position-$0$ attention, averaging over late-layer $\times$ head $\times$ trailing-query positions with bootstrap $95\%$ CIs over $100$ prompts per outcome bucket (Appendix~\ref{app:sink_outcome}).
The within-pair pass$-$fail gap is small ($\le 0.02$ per-prompt mean), but the cross-pair span is large.
The two per-head $\qknorm$ pairs concentrate $\ge\!0.70$ of attention on position $0$, while the no-$\qknorm$ pairs span $5\!\times\!10^{-5}$ to $0.50$ and Molmo2-O (layerwise) sits at $0.25$.
Within the same OpenGVLab InternVL lineage, InternVL3 (no-$\qknorm$, $5\!\times\!10^{-5}$) and InternVL3.5 (per-head $\qknorm$, $0.74$) span four orders of magnitude on this measure; the Qwen-VL family shows the same direction at smaller magnitude ($0.03 \to 0.70$ across Qwen2.5-VL $\to$ Qwen3-VL).
Within-vendor framing reduces vendor-level confounding and strengthens the association with the architectural change, although recipe differences remain.
The within-pair near-equality is consistent with the broken-sink
state being a model-level head configuration rather than a
prompt-level outcome marker.
The cross-pair span is what tracks the behavioral split, so effective interventions should operate at the head level.

\paragraph{Per-head vs.\ aggregate sink.}
\label{sec:method_perhead}
Molmo2-O (Olmo-3 base, layerwise $\qknorm$; Remark~\ref{rem:layerwise}) separates \emph{aggregate} from \emph{per-head} sink: aggregate position-$0$ attention ($0.25$) is of the same order as LLaVA-OV's ($0.50$) yet IFEval damage is far worse, while per-head sink is the weakest among the calibrated pairs and already weak in the Olmo-3 reference LLM
(Appendix~\ref{app:molmo}).
The marginal proxy $G - B$ is strongly negative ($-3.01$, second only to the no-$\qknorm$ Qwen2.5-VL at $-3.83$; Table~\ref{tab:G_Bgap}), so Lemma~\ref{lem:saturation}'s per-head margin protection is not delivered head-by-head.
MolmoE-1B, a second layerwise backbone on an MoE base, shows the same pattern, so both layerwise pairs carry strongly negative margins despite differing in base family and architecture.
Together, the evidence supports per-head $\qknorm$ as the protective configuration in our sample, with the layerwise variant insufficient.

\subsection{Channel and Direction Ablations}
\label{sec:method_ablation}

\paragraph{Ablating the sink amplifier.}
For each of Qwen3-VL's normalization-scale tensors
($\gamma_{\mathrm{in}}, \gamma_q, \gamma_k$), replacing the top-$10$ channels per layer by the layer-wise mean disables the amplifier without touching the residual stream or projections (Appendix~\ref{app:c1details}): joint kill costs 44 pt on IFEval against 20 pt as the sum of singles, a 24-pt super-additive interaction that supports sink amplification as a redundant multi-norm system contributing to format-sensitive capability.
For reference, a norm-matched isotropic random perturbation causes a $\sim6\times$ larger IFEval drop than the Qwen2.5-VL gap (Appendix~\ref{app:e2details}).
Thus, weight-space distance alone does not explain the degradation; perturbation direction matters.
\section{Controls and Discussion}
\label{sec:controls}

\paragraph{Within-vendor comparison.}
The same direction appears within each vendor's own model line: Qwen3-VL
vs.\ Qwen2.5-VL, and InternVL3.5 vs.\ InternVL3.
In both comparisons the $\qknorm$ backbone loses less text capability, so
the association does not require comparing models built by different teams
under different VL pipelines.
However, this is not a controlled comparison, since successive releases
differ in much more than $\qknorm$, and we report it as one alternative
explanation ruled out rather than as evidence that $\qknorm$ is the cause;
isolating the architectural axis requires the ablation we leave to future
work (See Limitations).

\paragraph{Modality control (VL vs.\ T\"ulu trajectory).}
\label{sec:modality_control}
The Section~\ref{sec:problem_setup} comparison contrasts released 7B endpoints against a common text-only reference.
We additionally run a matched 3B trajectory, comparing vision fine-tuning with text-only T\"ulu SFT from the same Qwen2.5-3B-Instruct base~\citep{lambert2024tulu}.
The vision run ends with a weaker sink ($S\!:\,+0.41\!\to\!+0.20$) as SEEDBench rises $60.1\!\to\!64.8$, whereas the text-only run holds $S$ at $+0.58$--$+0.65$ and loses only $\sim\!5$ IFEval points from the shared base (Appendix~\ref{app:vl_tulu}).
Because $S$ is a base-model property, it predicts the regime a run settles into rather than the step at which capability is lost.
IFEval takes most of its loss by step 1k, before $S$ declines, and then partially recovers while $S$ keeps falling.

\paragraph{Post-pretraining $\qknorm$ injection (negative).}
\label{sec:negative}
A natural question is whether the protection associated with native per-head $\qknorm$ can be reproduced by injecting the module after LLM pretraining.
We add per-head $\qknorm$ with unit-scale initialization to a clean Qwen2.5-3B-Instruct and then run the matched two-stage VL recipe (Appendix~\ref{app:c3details}).
The injected variant lags vanilla throughout training, scoring only 11 on IFEval at the first evaluated checkpoint against the 59.9 base.
Combined with the Molmo2-O control (Section~\ref{sec:method}), this supports native-pretraining co-adaptation of the projections, rather than the module's mere presence, as an important component of the protective regime.

\paragraph{Weight-space merging (negative).}
Merging the VLM's language backbone back toward its reference LLM is the natural post-hoc repair: seven off-the-shelf merging settings spanning linear task-vector mixes~\citep{ilharco2022editing}, spectral decompositions~\citep{gargiulo2025task,lee2025star}, and sparsified merges~\citep{yadav2023ties,yu2024language} on Qwen2.5-VL and Qwen3-VL fail to recover IFEval on the low-$S$ side: only the two mildest mixes avoid substantial additional degradation (Appendix~\ref{app:method_sweep}).

\paragraph{What the negative results do and do not show.}
Our matched T\"ulu trajectory provides a mechanistic rationale for the data-side remedy: text-only optimization preserves the sink whereas VL optimization disrupts it.
That remedy is applied uniformly and paid for in multimodal data budget, and it is chosen after the backbone is already fixed; $S$ instead identifies which backbones need it before training begins.
The two negative controls argue against simpler shortcuts:
post-pretraining $\qknorm$ injection does not reproduce native protection, and post-VL merging does not recover the lost capability under our settings.
Both leave the intervention space narrower than the remedy literature suggests.

\section{Conclusion}
\label{sec:conclusion}

Across five headline VLM--LLM pairs from three model families, format-sensitive text-capability gaps track Sink Strength of the corresponding reference LLM, with per-head $\qknorm$ as the strongest architectural correlate in our sample.
The relation extends from the six-pair diagnostic set to a 17-pair panel spanning MoE and multiple LLM families.
Negative injection and merging controls further motivate pre-VL screening and head-selective training-time protection.

\clearpage
\section*{Limitations}
\label{sec:limitations}

To our knowledge, the multimodal literature has not previously characterized the phenomenon we surface, in which VL adaptation erodes the base LLM's per-head attention sink and carries a downstream cost to format-sensitive capability.
Within our sample, per-head $\qknorm$ co-varies with vendor, recipe, and pretraining data.
Within-vendor comparisons (Qwen2.5$\rightarrow$Qwen3,
InternVL3$\rightarrow$InternVL3.5) reduce but do not isolate
the architectural axis.
We validate $S$ against VL adaptation outcomes, so whether it also predicts instruction-following drift under adaptation regimes we did not sample, such as continual pretraining, remains open and is a natural next test of the diagnostic's reach.

The post-VL merging methods we test do not recover the lost capability, while post-pretraining $\qknorm$ injection fails to reproduce the native protective regime, leaving head-selective training-time intervention as an open direction (Section~\ref{sec:controls}).
The bound of Section~\ref{sec:method} suggests one concrete candidate: freeze the layers below the first sink layer and project the sink-head $W_q, W_k$ updates away from the sink direction, which should slow how fast VL training closes the $G_{\mathrm{base}} - B_{\mathrm{gap}}$ margin.
A direct causal intervention would separately test sufficiency, for instance disabling per-head $\qknorm$ during VL training of Qwen3 to see whether IFEval collapses.
We leave both to future work.
\section*{Acknowledgments}
This work was supported by National Research Foundation of Korea (NRF) grant funded by the Korea government (MSIT) (RS-2025-24534857 and RS-2026-25522655),
Institute of Information \& Communications Technology Planning \& Evaluation(IITP)-ICT Creative Consilience Program grant funded by the Korea government (MSIT) (IITP-2026-RS-2020-II201819, 10\%),
IITP-ITRC(Information Technology Research Center) grant funded by the Korea government (MSIT) (IITP-2026-RS-2023-00260091, 10\%), and IITP grant funded by the Korea government (MSIT) (RS-2024-00398353, Development of Countermeasure Technologies for Generative AI Security Threats). We also thank the members of the Korea University Intelligent Computer Architecture \& Systems research Lab, Vision Understanding Team at NAVER Cloud AI, and the anonymous reviewers for their helpful feedback.

\bibliography{custom}

@article{zhang2024wings,
  title={{Wings: Learning Multimodal LLMs without Text-only Forgetting}},
  author={Zhang, Yi-Kai and Lu, Shiyin and Li, Yang and Ma, Yanqing and Chen, Qing-Guo and Xu, Zhao and Luo, Weihua and Zhang, Kaifu and Zhan, De-Chuan and Ye, Han-Jia},
  journal={Advances in Neural Information Processing Systems},
  volume={37},
  pages={31828--31853},
  year={2024}
}

@inproceedings{zhai2024investigating,
  title={{Investigating the Catastrophic Forgetting in Multimodal Large Language Model Fine-Tuning}},
  author={Zhai, Yuexiang and Tong, Shengbang and Li, Xiao and Cai, Mu and Qu, Qing and Lee, Yong Jae and Ma, Yi},
  booktitle={Conference on Parsimony and Learning},
  pages={202--227},
  volume={234},
  series={Proceedings of Machine Learning Research},
  publisher={PMLR},
  year={2024}
}

@inproceedings{lee2025does,
  title={{How Does Vision-Language Adaptation Impact the Safety of Vision Language Models?}},
  author={Lee, Seongyun and Kim, Geewook and Kim, Jiyeon and Lee, Hyunji and Chang, Hoyeon and Park, Sue Hyun and Seo, Minjoon},
  booktitle={International Conference on Learning Representations},
  year={2025}
}

@inproceedings{ratzlaff2025training,
  title={{Training-Free Mitigation of Language Reasoning Degradation After Multimodal Instruction Tuning}},
  author={Ratzlaff, Neale and Luo, Man and Su, Xin and Lal, Vasudev and Howard, Phillip},
  booktitle={Proceedings of the AAAI Symposium Series},
  volume={5},
  pages={384--388},
  year={2025}
}

@inproceedings{srivastava2024improving,
  title={{Improving Multimodal Large Language Models Using Continual Learning}},
  author={Srivastava, Shikhar and Harun, Md Yousuf and Shrestha, Robik Singh and Kanan, Christopher},
  booktitle={Proceedings of The 4th Conference on Lifelong Learning Agents},
  pages={736--755},
  volume={330},
  series={Proceedings of Machine Learning Research},
  publisher={PMLR},
  year={2026}
}

@inproceedings{lin2024vila,
  title={{VILA: On Pre-training for Visual Language Models}},
  author={Lin, Ji and Yin, Hongxu and Ping, Wei and Molchanov, Pavlo and Shoeybi, Mohammad and Han, Song},
  booktitle={Proceedings of the IEEE/CVF Conference on Computer Vision and Pattern Recognition (CVPR)},
  pages={26689--26699},
  year={2024}
}

@article{tong2024cambrian,
  title={{Cambrian-1: A Fully Open, Vision-Centric Exploration of Multimodal LLMs}},
  author={Tong, Shengbang and Brown, Ellis and Wu, Penghao and Woo, Sanghyun and Middepogu, Manoj and Akula, Sai Charitha and Yang, Jihan and Yang, Shusheng and Iyer, Adithya and Pan, Xichen and Wang, Austin and Fergus, Rob and LeCun, Yann and Xie, Saining},
  journal={Advances in Neural Information Processing Systems},
  volume={37},
  pages={87310--87356},
  year={2024}
}

@inproceedings{tu2025mlan,
  title={{MLAN: Language-Based Instruction Tuning Preserves and Transfers Knowledge in Multimodal Language Models}},
  author={Tu, Jianhong and Ni, Zhuohao and Crispino, Nicholas and Yu, Zihao and Bendersky, Michael and Gunel, Beliz and Jia, Ruoxi and Liu, Xin and Lyu, Lingjuan and Song, Dawn and Wang, Chenguang},
  booktitle={Proceedings of the 3rd Workshop on Towards Knowledgeable Foundation Models (KnowFM)},
  pages={59--74},
  year={2025}
}

@inproceedings{zhu2024minigpt,
  title={{MiniGPT-4: Enhancing Vision-Language Understanding with Advanced Large Language Models}},
  author={Zhu, Deyao and Chen, Jun and Shen, Xiaoqian and Li, Xiang and Elhoseiny, Mohamed},
  booktitle={International Conference on Learning Representations},
  year={2024}
}

@article{liu2023visual,
  title={{Visual Instruction Tuning}},
  author={Liu, Haotian and Li, Chunyuan and Wu, Qingyang and Lee, Yong Jae},
  journal={Advances in Neural Information Processing Systems},
  volume={36},
  pages={34892--34916},
  year={2023}
}

@inproceedings{liu2024improved,
  title={{Improved Baselines with Visual Instruction Tuning}},
  author={Liu, Haotian and Li, Chunyuan and Li, Yuheng and Lee, Yong Jae},
  booktitle={Proceedings of the IEEE/CVF Conference on Computer Vision and Pattern Recognition (CVPR)},
  pages={26296--26306},
  year={2024}
}

@inproceedings{hu2022lora,
  title={{LoRA: Low-Rank Adaptation of Large Language Models}},
  author={Hu, Edward J. and Shen, Yelong and Wallis, Phillip and Allen-Zhu, Zeyuan and Li, Yuanzhi and Wang, Shean and Wang, Lu and Chen, Weizhu},
  booktitle={International Conference on Learning Representations},
  year={2022}
}

@inproceedings{yu2025locate,
  title={{Locate-then-Merge: Neuron-Level Parameter Fusion for Mitigating Catastrophic Forgetting in Multimodal LLMs}},
  author={Yu, Zeping and Ananiadou, Sophia},
  booktitle={Findings of the Association for Computational Linguistics: EMNLP 2025},
  pages={7065--7078},
  year={2025}
}

@inproceedings{ilharco2022editing,
  title={{Editing models with task arithmetic}},
  author={Ilharco, Gabriel and Ribeiro, Marco Tulio and Wortsman, Mitchell and Gururangan, Suchin and Schmidt, Ludwig and Hajishirzi, Hannaneh and Farhadi, Ali},
  booktitle={International Conference on Learning Representations},
  year={2023}
}

@article{yadav2023ties,
  title={{TIES-Merging: Resolving Interference When Merging Models}},
  author={Yadav, Prateek and Tam, Derek and Choshen, Leshem and Raffel, Colin and Bansal, Mohit},
  journal={Advances in Neural Information Processing Systems},
  volume={36},
  pages={7093--7115},
  year={2023}
}

@inproceedings{yu2024language,
  title={{Language Models are Super Mario: Absorbing Abilities from Homologous Models as a Free Lunch}},
  author={Yu, Le and Yu, Bowen and Yu, Haiyang and Huang, Fei and Li, Yongbin},
  booktitle={Proceedings of the 41st International Conference on Machine Learning},
  pages={57755--57775},
  volume={235},
  series={Proceedings of Machine Learning Research},
  publisher={PMLR},
  year={2024}
}

@inproceedings{wortsman2022model,
  title={{Model soups: averaging weights of multiple fine-tuned models improves accuracy without increasing inference time}},
  author={Wortsman, Mitchell and Ilharco, Gabriel and Gadre, Samir Ya and Roelofs, Rebecca and Gontijo-Lopes, Raphael and Morcos, Ari S and Namkoong, Hongseok and Farhadi, Ali and Carmon, Yair and Kornblith, Simon and Schmidt, Ludwig},
  booktitle={Proceedings of the 39th International Conference on Machine Learning},
  pages={23965--23998},
  volume={162},
  series={Proceedings of Machine Learning Research},
  publisher={PMLR},
  year={2022}
}

@inproceedings{lee2025star,
  title={{STAR: Spectral Truncation and Rescale for Model Merging}},
  author={Lee, Yu-Ang and Ko, Ching-Yun and Pedapati, Tejaswini and Chung, I-Hsin and Yeh, Mi-Yen and Chen, Pin-Yu},
  booktitle={Proceedings of the 2025 Conference of the Nations of the Americas Chapter of the Association for Computational Linguistics: Human Language Technologies (Volume 2: Short Papers)},
  pages={496--505},
  year={2025}
}

@inproceedings{gargiulo2025task,
  title={{Task Singular Vectors: Reducing Task Interference in Model Merging}},
  author={Gargiulo, Antonio Andrea and Crisostomi, Donato and Bucarelli, Maria Sofia and Scardapane, Simone and Silvestri, Fabrizio and Rodol\`a, Emanuele},
  booktitle={Proceedings of the IEEE/CVF Conference on Computer Vision and Pattern Recognition (CVPR)},
  pages={18695--18705},
  year={2025}
}

@misc{eval-harness,
  author       = {Gao, Leo and Tow, Jonathan and Abbasi, Baber and Biderman, Stella and Black, Sid and DiPofi, Anthony and Foster, Charles and Golding, Laurence and Hsu, Jeffrey and Le Noac'h, Alain and Li, Haonan and McDonell, Kyle and Muennighoff, Niklas and Ociepa, Chris and Phang, Jason and Reynolds, Laria and Schoelkopf, Hailey and Skowron, Aviya and Sutawika, Lintang and Tang, Eric and Thite, Anish and Wang, Ben and Wang, Kevin and Zou, Andy},
  title        = {{The Language Model Evaluation Harness}},
  month        = 07,
  year         = 2024,
  publisher    = {Zenodo},
  version      = {v0.4.3},
  doi          = {10.5281/zenodo.12608602},
  url          = {https://zenodo.org/records/12608602}
}

@article{zhou2023instruction,
  title={{Instruction-Following Evaluation for Large Language Models}},
  author={Zhou, Jeffrey and Lu, Tianjian and Mishra, Swaroop and Brahma, Siddhartha and Basu, Sujoy and Luan, Yi and Zhou, Denny and Hou, Le},
  journal={arXiv preprint arXiv:2311.07911},
  year={2023}
}

@article{yang2024qwen2technicalreport,
  title={{Qwen2 Technical Report}},
  author={Yang, An and Yang, Baosong and Hui, Binyuan and Zheng, Bo and Yu, Bowen and Zhou, Chang and Li, Chengpeng and Li, Chengyuan and Liu, Dayiheng and Huang, Fei and others},
  journal={arXiv preprint arXiv:2407.10671},
  year={2024}
}

@article{qwen2025qwen25technicalreport,
  title={{Qwen2.5 Technical Report}},
  author={Yang, An and Yang, Baosong and Zhang, Beichen and Hui, Binyuan and Zheng, Bo and Yu, Bowen and Li, Chengyuan and Liu, Dayiheng and Huang, Fei and Wei, Haoran and others},
  journal={arXiv preprint arXiv:2412.15115},
  year={2024}
}

@article{grattafiori2024llama,
  title={{The Llama 3 Herd of Models}},
  author={Grattafiori, Aaron and Dubey, Abhimanyu and Jauhri, Abhinav and Pandey, Abhinav and Kadian, Abhishek and Al-Dahle, Ahmad and Letman, Aiesha and Mathur, Akhil and Schelten, Alan and Vaughan, Alex and others},
  journal={arXiv preprint arXiv:2407.21783},
  year={2024}
}

@article{yu2025minicpmv45,
  title={{MiniCPM-V 4.5: Cooking Efficient MLLMs via Architecture, Data, and Training Recipe}},
  author={Yu, Tianyu and Wang, Zefan and Wang, Chongyi and Huang, Fuwei and Ma, Wenshuo and He, Zhihui and Cai, Tianchi and Chen, Weize and Huang, Yuxiang and Zhao, Yuanqian and others},
  journal={arXiv preprint arXiv:2509.18154},
  year={2025}
}

@article{yang2025qwen3,
  title={{Qwen3 Technical Report}},
  author={Yang, An and Li, Anfeng and Yang, Baosong and Zhang, Beichen and Hui, Binyuan and Zheng, Bo and Yu, Bowen and Gao, Chang and Huang, Chengen and Lv, Chenxu and others},
  journal={arXiv preprint arXiv:2505.09388},
  year={2025}
}

@inproceedings{xiao2024efficient,
  title={{Efficient Streaming Language Models with Attention Sinks}},
  author={Xiao, Guangxuan and Tian, Yuandong and Chen, Beidi and Han, Song and Lewis, Mike},
  booktitle={International Conference on Learning Representations},
  year={2024}
}

@inproceedings{sun2024massive,
  title={{Massive Activations in Large Language Models}},
  author={Sun, Mingjie and Chen, Xinlei and Kolter, J. Zico and Liu, Zhuang},
  booktitle={First Conference on Language Modeling},
  year={2024}
}

@inproceedings{gu2025attention,
  title={{When Attention Sink Emerges in Language Models: An Empirical View}},
  author={Gu, Xiangming and Pang, Tianyu and Du, Chao and Liu, Qian and Zhang, Fengzhuo and Du, Cunxiao and Wang, Ye and Lin, Min},
  booktitle={International Conference on Learning Representations},
  year={2025}
}

@inproceedings{barbero2025llms,
  title={{Why do LLMs attend to the first token?}},
  author={Barbero, Federico and Arroyo, {\'A}lvaro and Gu, Xiangming and Perivolaropoulos, Christos and Bronstein, Michael and Veli{\v{c}}kovi{\'c}, Petar and Pascanu, Razvan},
  booktitle={Second Conference on Language Modeling},
  year={2025}
}

@article{zhang2025attention,
  title={{Attention Sinks: A 'Catch, Tag, Release' Mechanism for Embeddings}},
  author={Zhang, Stephen and Khan, Mustafa and Papyan, Vardan},
  journal={Advances in Neural Information Processing Systems},
  volume={38},
  pages={83140--83181},
  year={2025}
}

@inproceedings{darcet2024vision,
  title={{Vision Transformers Need Registers}},
  author={Darcet, Timoth{\'e}e and Oquab, Maxime and Mairal, Julien and Bojanowski, Piotr},
  booktitle={International Conference on Learning Representations},
  year={2024}
}

@inproceedings{henry2020query,
  title={{Query-Key Normalization for Transformers}},
  author={Henry, Alex and Dachapally, Prudhvi Raj and Pawar, Shubham Shantaram and Chen, Yuxuan},
  booktitle={Findings of the Association for Computational Linguistics: EMNLP 2020},
  pages={4246--4253},
  year={2020}
}

@inproceedings{zhai2023stabilizing,
  title={{Stabilizing Transformer Training by Preventing Attention Entropy Collapse}},
  author={Zhai, Shuangfei and Likhomanenko, Tatiana and Littwin, Etai and Busbridge, Dan and Ramapuram, Jason and Zhang, Yizhe and Gu, Jiatao and Susskind, Joshua M.},
  booktitle={Proceedings of the 40th International Conference on Machine Learning},
  pages={40770--40803},
  volume={202},
  series={Proceedings of Machine Learning Research},
  publisher={PMLR},
  year={2023}
}

@article{olsson2022context,
  title={{In-context Learning and Induction Heads}},
  author={Olsson, Catherine and Elhage, Nelson and Nanda, Neel and Joseph, Nicholas and DasSarma, Nova and Henighan, Tom and Mann, Ben and Askell, Amanda and Bai, Yuntao and Chen, Anna and others},
  journal={arXiv preprint arXiv:2209.11895},
  year={2022}
}

@inproceedings{geva2021transformer,
  title={{Transformer Feed-Forward Layers Are Key-Value Memories}},
  author={Geva, Mor and Schuster, Roei and Berant, Jonathan and Levy, Omer},
  booktitle={Proceedings of the 2021 Conference on Empirical Methods in Natural Language Processing},
  pages={5484--5495},
  year={2021}
}

@inproceedings{wang2025towards,
  title={{Towards Understanding Fine-Tuning Mechanisms of LLMs via Circuit Analysis}},
  author={Wang, Xu and Hu, Yan and Du, Wenyu and Cheng, Reynold and Wang, Benyou and Zou, Difan},
  booktitle={Proceedings of the 42nd International Conference on Machine Learning},
  pages={63088--63112},
  volume={267},
  series={Proceedings of Machine Learning Research},
  publisher={PMLR},
  year={2025}
}

@article{paech2023eq,
  title={{EQ-Bench: An Emotional Intelligence Benchmark for Large Language Models}},
  author={Paech, Samuel J.},
  journal={arXiv preprint arXiv:2312.06281},
  year={2023}
}

@article{cobbe2021training,
  title={{Training Verifiers to Solve Math Word Problems}},
  author={Cobbe, Karl and Kosaraju, Vineet and Bavarian, Mohammad and Chen, Mark and Jun, Heewoo and Kaiser, Lukasz and Plappert, Matthias and Tworek, Jerry and Hilton, Jacob and Nakano, Reiichiro and Hesse, Christopher and Schulman, John},
  journal={arXiv preprint arXiv:2110.14168},
  year={2021}
}

@inproceedings{rein2023gpqa,
  title={{GPQA: A Graduate-Level Google-Proof {Q\&A} Benchmark}},
  author={Rein, David and Hou, Betty Li and Stickland, Asa Cooper and Petty, Jackson and Pang, Richard Yuanzhe and Dirani, Julien and Michael, Julian and Bowman, Samuel R},
  booktitle={First Conference on Language Modeling},
  year={2024}
}

@inproceedings{hendrycks2020measuring,
  title={{Measuring Massive Multitask Language Understanding}},
  author={Hendrycks, Dan and Burns, Collin and Basart, Steven and Zou, Andy and Mazeika, Mantas and Song, Dawn and Steinhardt, Jacob},
  booktitle={International Conference on Learning Representations},
  year={2021}
}

@inproceedings{clark2019boolq,
  title={{BoolQ: Exploring the Surprising Difficulty of Natural Yes/No Questions}},
  author={Clark, Christopher and Lee, Kenton and Chang, Ming-Wei and Kwiatkowski, Tom and Collins, Michael and Toutanova, Kristina},
  booktitle={Proceedings of the 2019 Conference of the North American Chapter of the Association for Computational Linguistics: Human Language Technologies, Volume 1 (Long and Short Papers)},
  pages={2924--2936},
  year={2019}
}

@article{bai2025qwen3,
  title={{Qwen3-VL Technical Report}},
  author={Bai, Shuai and Cai, Yuxuan and Chen, Ruizhe and Chen, Keqin and Chen, Xionghui and Cheng, Zesen and Deng, Lianghao and Ding, Wei and Gao, Chang and Ge, Chunjiang and others},
  journal={arXiv preprint arXiv:2511.21631},
  year={2025}
}

@article{wang2025internvl3,
  title={{InternVL3.5: Advancing Open-Source Multimodal Models in Versatility, Reasoning, and Efficiency}},
  author={Wang, Weiyun and Gao, Zhangwei and Gu, Lixin and Pu, Hengjun and Cui, Long and Wei, Xingguang and Liu, Zhaoyang and Jing, Linglin and Ye, Shenglong and Shao, Jie and others},
  journal={arXiv preprint arXiv:2508.18265},
  year={2025}
}

@article{bai2025qwen25vltechnicalreport,
  title={{Qwen2.5-VL Technical Report}},
  author={Bai, Shuai and Chen, Keqin and Liu, Xuejing and Wang, Jialin and Ge, Wenbin and Song, Sibo and Dang, Kai and Wang, Peng and Wang, Shijie and Tang, Jun and others},
  journal={arXiv preprint arXiv:2502.13923},
  year={2025}
}

@article{zhu2025internvl3,
  title={{InternVL3: Exploring Advanced Training and Test-Time Recipes for Open-Source Multimodal Models}},
  author={Zhu, Jinguo and Wang, Weiyun and Chen, Zhe and Liu, Zhaoyang and Ye, Shenglong and Gu, Lixin and Tian, Hao and Duan, Yuchen and Su, Weijie and Shao, Jie and others},
  journal={arXiv preprint arXiv:2504.10479},
  year={2025}
}

@article{li2024llava,
  title={{LLaVA-OneVision: Easy Visual Task Transfer}},
  author={Li, Bo and Zhang, Yuanhan and Guo, Dong and Zhang, Renrui and Li, Feng and Zhang, Hao and Zhang, Kaichen and Zhang, Peiyuan and Li, Yanwei and Liu, Ziwei and Li, Chunyuan},
  journal={Transactions on Machine Learning Research},
  year={2025}
}

@inproceedings{deitke2025molmo,
  title={{Molmo and PixMo: Open Weights and Open Data for State-of-the-Art Vision-Language Models}},
  author={Deitke, Matt and Clark, Christopher and Lee, Sangho and Tripathi, Rohun and Yang, Yue and Park, Jae Sung and Salehi, Mohammadreza and Muennighoff, Niklas and Lo, Kyle and Soldaini, Luca and others},
  booktitle={Proceedings of the IEEE/CVF Conference on Computer Vision and Pattern Recognition (CVPR)},
  pages={91--104},
  year={2025}
}

@inproceedings{clark2026molmo2,
  title={{Molmo2: Open Weights and Data for Vision-Language Models with Video Understanding and Grounding}},
  author={Clark, Christopher and Zhang, Jieyu and Ma, Zixian and Park, Jae Sung and Tripathi, Rohun and Lee, Sangho and Salehi, Mohammadreza and Ren, Jason and Kim, Chris Dongjoo and Yang, Yinuo and others},
  booktitle={Proceedings of the IEEE/CVF Conference on Computer Vision and Pattern Recognition (CVPR)},
  pages={28652--28668},
  year={2026}
}

@article{lu2024ovis,
  title={{Ovis: Structural Embedding Alignment for Multimodal Large Language Model}},
  author={Lu, Shiyin and Li, Yang and Chen, Qing-Guo and Xu, Zhao and Luo, Weihua and Zhang, Kaifu and Ye, Han-Jia},
  journal={arXiv preprint arXiv:2405.20797},
  year={2024}
}

@inproceedings{laurenccon2024building,
  title={{Building and better understanding vision-language models: insights and future directions}},
  author={Lauren{\c{c}}on, Hugo and Marafioti, Andr{\'e}s and Sanh, Victor and Tronchon, L{\'e}o},
  booktitle={NeurIPS 2024 Workshops: RBFM},
  year={2024}
}

@inproceedings{muennighoff2025olmoe,
  title={{OLMoE: Open Mixture-of-Experts Language Models}},
  author={Muennighoff, Niklas and Soldaini, Luca and Groeneveld, Dirk and Lo, Kyle and Morrison, Jacob and Min, Sewon and Shi, Weijia and Walsh, Pete and Tafjord, Oyvind and Lambert, Nathan and others},
  booktitle={International Conference on Learning Representations},
  year={2025}
}

@misc{jiang2023mistral7b,
  title={{Mistral 7B}},
  author={Jiang, Albert Q. and Sablayrolles, Alexandre and Mensch, Arthur and Bamford, Chris and Chaplot, Devendra Singh and de las Casas, Diego and Bressand, Florian and Lengyel, Gianna and Lample, Guillaume and Saulnier, Lucile and others},
  year={2023},
  eprint={2310.06825},
  archivePrefix={arXiv},
  primaryClass={cs.CL},
  url={https://arxiv.org/abs/2310.06825}
}

@article{austin2021program,
  title={{Program Synthesis with Large Language Models}},
  author={Austin, Jacob and Odena, Augustus and Nye, Maxwell and Bosma, Maarten and Michalewski, Henryk and Dohan, David and Jiang, Ellen and Cai, Carrie and Terry, Michael and Le, Quoc and Sutton, Charles},
  journal={arXiv preprint arXiv:2108.07732},
  year={2021}
}

@article{wang2024mmlu,
  title={{MMLU-Pro: A More Robust and Challenging Multi-Task Language Understanding Benchmark}},
  author={Wang, Yubo and Ma, Xueguang and Zhang, Ge and Ni, Yuansheng and Chandra, Abhranil and Guo, Shiguang and Ren, Weiming and Arulraj, Aaran and He, Xuan and Jiang, Ziyan and others},
  journal={Advances in Neural Information Processing Systems},
  volume={37},
  pages={95266--95290},
  year={2024}
}

@inproceedings{hsieh2024ruler,
  title={{RULER: What's the Real Context Size of Your Long-Context Language Models?}},
  author={Hsieh, Cheng-Ping and Sun, Simeng and Kriman, Samuel and Acharya, Shantanu and Rekesh, Dima and Jia, Fei and Zhang, Yang and Ginsburg, Boris},
  booktitle={First Conference on Language Modeling},
  year={2024}
}

@inproceedings{lambert2024tulu,
  title={{T{\"u}lu 3: Pushing Frontiers in Open Language Model Post-Training}},
  author={Lambert, Nathan and Morrison, Jacob and Pyatkin, Valentina and Huang, Shengyi and Ivison, Hamish and Brahman, Faeze and Miranda, Lester James V. and Liu, Alisa and Dziri, Nouha and Lyu, Xinxi and others},
  booktitle={Second Conference on Language Modeling},
  year={2025}
}

@article{li2023seedbench,
  title={{SEED-Bench: Benchmarking Multimodal LLMs with Generative Comprehension}},
  author={Li, Bohao and Wang, Rui and Wang, Guangzhi and Ge, Yuying and Ge, Yixiao and Shan, Ying},
  journal={arXiv preprint arXiv:2307.16125},
  year={2023}
}

@inproceedings{choi2026decentralized,
  title     = {{Decentralized Instruction Tuning: Conflict-Aware Splitting and Weight Merging}},
  author    = {Choi, Minsik and Kim, Geewook},
  booktitle = {International Conference on Machine Learning},
  year      = {2026}
}

@article{olmo2025olmo,
  title={{Olmo 3}},
  author={{Team Olmo} and Ettinger, Allyson and Bertsch, Amanda and Kuehl, Bailey and Graham, David and Heineman, David and Groeneveld, Dirk and Brahman, Faeze and Timbers, Finbarr and Ivison, Hamish and others},
  journal={arXiv preprint arXiv:2512.13961},
  year={2025}
}

\appendix
\onecolumn

\noindent\textbf{Table of Contents}

\begin{tabularx}{\linewidth}{@{}X r@{}}
\toprule
\hyperref[app:proofs]{A. Theory and Proofs} \dotfill & p.~\pageref{app:proofs} \\
\quad \hyperref[app:proof_thmA]{A.1. Single-Head Perturbation Bound (Theorem~\ref{thm:perturb})} \dotfill & p.~\pageref{app:proof_thmA} \\
\quad \hyperref[app:proof_lemB]{A.2. Aggregate-Gap Saturation Transfer (Lemma~\ref{lem:saturation})} \dotfill & p.~\pageref{app:proof_lemB} \\
\quad \hyperref[app:stable_rank]{A.3. Frobenius / Stable-Rank Restatement} \dotfill & p.~\pageref{app:stable_rank} \\
\quad \hyperref[app:multi_layer]{A.4. Multi-Layer Composition Remark} \dotfill & p.~\pageref{app:multi_layer} \\
\quad \hyperref[app:layerwise_deriv]{A.5. Why Layerwise QK-RMSNorm Does Not Provide the Same Per-Head Guarantee} \dotfill & p.~\pageref{app:layerwise_deriv} \\
\quad \hyperref[app:numerical]{A.6. Numerical Sanity Checks} \dotfill & p.~\pageref{app:numerical} \\
\hyperref[app:setup]{B. Evaluation Setup} \dotfill & p.~\pageref{app:setup} \\
\quad \hyperref[sec:protocol]{B.1. Evaluation Protocols} \dotfill & p.~\pageref{sec:protocol} \\
\quad \hyperref[sec:extraction]{B.2. Backbone Extraction} \dotfill & p.~\pageref{sec:extraction} \\
\quad \hyperref[sec:impl]{B.3. Implementation Details and Reproducibility} \dotfill & p.~\pageref{sec:impl} \\
\quad \hyperref[app:datasets]{B.4. Models and Datasets} \dotfill & p.~\pageref{app:datasets} \\
\hyperref[app:behavior]{C. Per-Task and Behavioral Evidence} \dotfill & p.~\pageref{app:behavior} \\
\quad \hyperref[app:knowledge_coverage]{C.1. Capability Preservation vs.\ Format-Sensitive Degradation} \dotfill & p.~\pageref{app:knowledge_coverage} \\
\quad \hyperref[app:failure_categories]{C.2. Failure-Category Breakdown} \dotfill & p.~\pageref{app:failure_categories} \\
\quad \hyperref[app:sink_outcome]{C.3. Outcome-Bucketed Sink Probe} \dotfill & p.~\pageref{app:sink_outcome} \\
\quad \hyperref[app:extended_bench]{C.4. Extended Capability Panel} \dotfill & p.~\pageref{app:extended_bench} \\
\hyperref[app:mechanism]{D. Mechanism Details} \dotfill & p.~\pageref{app:mechanism} \\
\quad \hyperref[app:calib_sweep]{D.1. Calibration Sample-Size Robustness} \dotfill & p.~\pageref{app:calib_sweep} \\
\quad \hyperref[app:molmo]{D.2. Layerwise Backbones: Aggregate vs.\ Per-Head Sink} \dotfill & p.~\pageref{app:molmo} \\
\quad \hyperref[app:c1details]{D.3. Channel Ablation Sweep} \dotfill & p.~\pageref{app:c1details} \\
\quad \hyperref[app:e2details]{D.4. Random-$\Delta W$ Control} \dotfill & p.~\pageref{app:e2details} \\
\hyperref[app:predictor_supp]{E. Sink Strength Predictor} \dotfill & p.~\pageref{app:predictor_supp} \\
\quad \hyperref[app:predictor_detail]{E.1. Per-Pair Sink Strength Predictions} \dotfill & p.~\pageref{app:predictor_detail} \\
\quad \hyperref[app:predictor_17pair]{E.2. IFEval-Specific Predictor (17 pairs)} \dotfill & p.~\pageref{app:predictor_17pair} \\
\quad \hyperref[app:predictor_extended]{E.3. Multi-Task Generalization (9 pairs)} \dotfill & p.~\pageref{app:predictor_extended} \\
\quad \hyperref[app:sink_design_ablations]{E.4. Sink Strength Design Ablations} \dotfill & p.~\pageref{app:sink_design_ablations} \\
\hyperref[app:controls]{F. Controls} \dotfill & p.~\pageref{app:controls} \\
\quad \hyperref[app:lm_sft]{F.1. Modality Comparison: Text-Only Endpoint} \dotfill & p.~\pageref{app:lm_sft} \\
\quad \hyperref[app:vl_tulu]{F.2. Modality Control: VL vs.\ Text-Only Trajectory} \dotfill & p.~\pageref{app:vl_tulu} \\
\quad \hyperref[app:c3details]{F.3. Injection Experiment Training Recipe} \dotfill & p.~\pageref{app:c3details} \\
\quad \hyperref[app:method_sweep]{F.4. Weight-Merging Sweep on Format-Sensitive Benches} \dotfill & p.~\pageref{app:method_sweep} \\
\bottomrule
\end{tabularx}

\clearpage

\section{Theory and Proofs}
\label{app:proofs}

Notation as in Section~\ref{sec:method}.
We work with a single attention head, hidden size $d$, head dim $d_h$.
The pre-norm inputs are $\xi_q, \xi_p = \gamma_{\mathrm{in}} \odot \rmsnorm(x_q), \gamma_{\mathrm{in}} \odot \rmsnorm(x_p) \in \mathbb{R}^d$ at the query and key-position-$p$ tokens, and the pre-$\qknorm$ projections are $Q_{\mathrm{pre}} = W_q \xi_q$, $K_{\mathrm{pre},p} = W_k \xi_p$ with $W_q, W_k \in \mathbb{R}^{d_h \times d}$.
Under $\qknorm$ we set $\tilde q = Q_{\mathrm{pre}} / \mathrm{RMS}(Q_{\mathrm{pre}})$ and $Q = \gamma_q \odot \tilde q$ (similarly $K_p$).
For clarity, the main text omits RoPE from the notation.
Here, let $R_q$ and $R_p$ denote the orthogonal RoPE transformations at the query and key positions.
Orthogonality preserves the norm of each rotated vector, while differences across positions are bounded by the triangle inequality; this yields the sum-of-norms term in Equation~\ref{eq:thm_perturb_no}.
The sink-anchored gap perturbation is $B_{\mathrm{gap}} := \sup_{p \ne p_{\mathrm{sink}}} |\delta_p - \delta_{p_{\mathrm{sink}}}|$, and the aggregate sink gap is $G := l_{p_{\mathrm{sink}}} - \log \sum_{p \ne p_{\mathrm{sink}}} e^{l_p}$.

\subsection{Single-Head Perturbation Bound (Theorem~\ref{thm:perturb})}
\label{app:proof_thmA}

\begin{proof}[Proof of Theorem~\ref{thm:perturb}]
Let $\Delta W_q, \Delta W_k$ be the projection perturbations with $\opnorm{\Delta W_*} \le \varepsilon$.
The two cases are treated separately because the dependence of $Q, K_p$ on the projections is linear without $\qknorm$ and nonlinear (through the RMS denominator) with $\qknorm$: the no-$\qknorm$ case is exactly linear up to the explicit second-order cross term, whereas the $\qknorm$ case requires a local Taylor expansion through the RMS denominator.

\paragraph{No $\qknorm$.}
Let $R_q$ and $R_p$ denote the orthogonal RoPE transformations
at the query position and key position $p$, respectively.
The attention logit is
\[
\ell_p
=
\frac{(R_q W_q\xi_q)^\top (R_p W_k\xi_p)}{\sqrt{d_h}}.
\]
Since the projections are linear, the first-order perturbation is
\[
\delta_p
=
\frac{1}{\sqrt{d_h}}
\Big[
(R_q\Delta W_q\xi_q)^\top(R_pW_k\xi_p)
+
(R_qW_q\xi_q)^\top(R_p\Delta W_k\xi_p)
\Big]
+
O(\varepsilon^2).
\]

Let $s=p_{\mathrm{sink}}$.
The sink-anchored gap perturbation between $s$ and any $p\neq s$ is
\begin{align}
\delta_p-\delta_s
=
\frac{1}{\sqrt{d_h}}
\Big[
&(R_q\Delta W_q\xi_q)^\top
\big(R_pW_k\xi_p-R_sW_k\xi_s\big)
\nonumber\\
&+
(R_qW_q\xi_q)^\top
\big(R_p\Delta W_k\xi_p-R_s\Delta W_k\xi_s\big)
\Big]
+
O(\varepsilon^2).
\end{align}

Because each $R_p$ is orthogonal,
\begin{align}
\|R_pW_k\xi_p-R_sW_k\xi_s\|
&\le
\|W_k\|_{\mathrm{op}}
\big(\|\xi_p\|+\|\xi_s\|\big),\\
\|R_p\Delta W_k\xi_p-R_s\Delta W_k\xi_s\|
&\le
\varepsilon
\big(\|\xi_p\|+\|\xi_s\|\big).
\end{align}
Also,
\[
\|R_q\Delta W_q\xi_q\|
\le \varepsilon\|\xi_q\|,
\qquad
\|R_qW_q\xi_q\|
\le \|W_q\|_{\mathrm{op}}\|\xi_q\|.
\]
Applying Cauchy--Schwarz therefore gives
\[
|\delta_p-\delta_s|
\le
\frac{\varepsilon}{\sqrt{d_h}}
\|\xi_q\|
\big(\|W_q\|_{\mathrm{op}}+\|W_k\|_{\mathrm{op}}\big)
\big(\|\xi_s\|+\|\xi_p\|\big)
+
O(\varepsilon^2).
\]
Taking the supremum over $p\neq s$ and absorbing fixed numerical
factors into the absolute constant $c$ yields
Equation~\ref{eq:thm_perturb_no}.

\paragraph{With $\qknorm$.}
The logit is $l_p = (\gamma_q \odot \tilde q) \cdot (\gamma_k \odot \tilde k_p)/\sqrt{d_h}$ where $\tilde q = Q_{\mathrm{pre}}/\mathrm{RMS}(Q_{\mathrm{pre}})$ and analogously $\tilde k_p$.
The map from $W_q, W_k$ to $\tilde q, \tilde k_p$ is differentiable on the open set where $\mathrm{RMS}(Q_{\mathrm{pre}}), \mathrm{RMS}(K_{\mathrm{pre},p}) > 0$, which Assumption~\ref{ass:rms} ensures.
We work with the mathematical RMS here, and production implementations add a small stabilizer $\tau > 0$ inside the square root (distinct from the perturbation magnitude $\varepsilon$).
Since $\mathrm{RMS}_\tau \ge \mathrm{RMS}$, the unstabilized bound
used below is conservative.
To first order in $\varepsilon$ on a small neighborhood the logit perturbation decomposes into a Q-side and K-side term:
\[
  \delta_p = \tfrac{1}{\sqrt{d_h}}\bigl[
    (\gamma_q \odot \Delta \tilde q) \cdot (\gamma_k \odot \tilde k_p)
    + (\gamma_q \odot \tilde q) \cdot (\gamma_k \odot \Delta \tilde k_p)
  \bigr] + O(\varepsilon^2).
\]
We bound each term separately.
For the Q-side perturbation,
\[
  \Delta \tilde q
  = \frac{\Delta W_q \xi_q}{\mathrm{RMS}(Q_{\mathrm{pre}})}
    - \frac{Q_{\mathrm{pre}}\,\Delta \mathrm{RMS}}{\mathrm{RMS}^2(Q_{\mathrm{pre}})},
\]
with $|\Delta \mathrm{RMS}| \le \opnorm{\Delta W_q}\|\xi_q\|/\sqrt{d_h}$.
Both terms are $O(\varepsilon \|\xi_q\| / \mathrm{RMS}(Q_{\mathrm{pre}}))$, and by Assumption~\ref{ass:rms} $\mathrm{RMS}(Q_{\mathrm{pre}}) \ge m_q \|\xi_q\|$ on the calibration distribution, so
\[
  \|\Delta \tilde q\| \le \frac{c_1 \varepsilon}{m_q},
\]
with $c_1$ an absolute constant.
The crucial point is that $\|\xi_q\|$ has cancelled, since numerator and denominator both scale linearly with $\|\xi_q\|$.

For the sink-anchored gap perturbation we then bound the Q-side
contribution after RoPE. Since each $R_p$ is orthogonal and
mathematical unit-RMS gives
$\|\tilde k_p\|_2=\sqrt{d_h}$ (with the stabilizer $\tau$,
$\|\tilde k_p\|_2\le\sqrt{d_h}$), we have
\[
\big\|
R_{p_{\mathrm{sink}}}
(\gamma_k\odot\tilde k_{p_{\mathrm{sink}}})
-
R_p(\gamma_k\odot\tilde k_p)
\big\|
\le
2\|\gamma_k\|_\infty\sqrt{d_h}.
\]
Therefore,
\begin{align*}
  &\Big|
  R_q(\gamma_q \odot \Delta \tilde q) \cdot
  \Big[
  R_{p_{\mathrm{sink}}}
  (\gamma_k \odot \tilde k_{p_{\mathrm{sink}}})
  -
  R_p(\gamma_k \odot \tilde k_p)
  \Big]
  \Big| \\
  &\le
  \|\gamma_q\|_\infty
  \|\Delta\tilde q\|
  \cdot
  2\|\gamma_k\|_\infty\sqrt{d_h} \\
  &\le
  \frac{2c_1\varepsilon
  \|\gamma_q\|_\infty\|\gamma_k\|_\infty
  \sqrt{d_h}}{m_q}.
\end{align*}
A symmetric argument for the K-side gives the $1/m_k$ term.
Summing the two sides, dividing by the outer $\sqrt{d_h}$ in the logit definition, and taking the supremum over $p \ne p_{\mathrm{sink}}$ yields
\[
  B_{\mathrm{gap}}
  \le c\,\varepsilon\,\|\gamma_q\|_\infty\,\|\gamma_k\|_\infty
      \bigl(\tfrac{1}{m_q} + \tfrac{1}{m_k}\bigr)
      + O(\varepsilon^2),
\]
which is Equation~\ref{eq:thm_perturb_yes}.
\end{proof}

\paragraph{Input-magnitude cancellation.}
The $\|\xi_q\|$ that the no-$\qknorm$ bound carries is cancelled in the $\qknorm$ pathway by the RMS denominator.
What replaces it is the projection-RMS lower bound $m_q$, and if a calibration prompt puts $\xi_q$ in or near the null space of $W_q$, $m_q$ degrades and the bound loosens.
Assumption~\ref{ass:rms} excludes exact degeneracy; near-null cases appear as small $m_q$ or $m_k$ and therefore loosen the bound, consistent with the empirical $m_q, m_k > 0$ we measure on the calibration set.

\subsection{Aggregate-Gap Saturation Transfer (Lemma~\ref{lem:saturation})}
\label{app:proof_lemB}

\begin{proof}[Proof of Lemma~\ref{lem:saturation}]
Write the perturbed softmax probability at position $p$ as $a'_p = e^{l'_p} / Z'$ with $l'_p = l_p + \delta_p$ and $Z' = \sum_q e^{l'_q}$.
Then for $p \ne p_{\mathrm{sink}}$,
\[
  \frac{a'_p}{a'_{p_{\mathrm{sink}}}}
  = \exp(l_p - l_{p_{\mathrm{sink}}} + \delta_p - \delta_{p_{\mathrm{sink}}}).
\]
Summing over $p \ne p_{\mathrm{sink}}$,
\begin{align*}
  \frac{1 - a'_{p_{\mathrm{sink}}}}{a'_{p_{\mathrm{sink}}}}
  &= \sum_{p \ne p_{\mathrm{sink}}} \exp\bigl(l_p - l_{p_{\mathrm{sink}}}\bigr)
     \exp\bigl(\delta_p - \delta_{p_{\mathrm{sink}}}\bigr) \\
  &\le \exp\bigl(B_{\mathrm{gap}}\bigr)
       \sum_{p \ne p_{\mathrm{sink}}} \exp\bigl(l_p - l_{p_{\mathrm{sink}}}\bigr) \\
  &= \exp\bigl(B_{\mathrm{gap}} - G\bigr),
\end{align*}
using the definition $G = l_{p_{\mathrm{sink}}} - \log\sum_{p \ne p_{\mathrm{sink}}} e^{l_p}$ and $|\delta_p - \delta_{p_{\mathrm{sink}}}| \le B_{\mathrm{gap}}$ for all $p \ne p_{\mathrm{sink}}$.
Let $u := 1 - a'_{p_{\mathrm{sink}}}$ and write the inequality as $u/(1-u) \le e^{B_{\mathrm{gap}} - G}$.
Solving for $u$ gives $u \le e^{B_{\mathrm{gap}} - G}/(1 + e^{B_{\mathrm{gap}} - G}) = \sigma(B_{\mathrm{gap}} - G)$, which is the logistic form of Equation~\ref{eq:lem_sat}.
The weaker exponential bound $u \le e^{B_{\mathrm{gap}} - G}$ follows from $\sigma(x) \le e^x$.
\end{proof}

\paragraph{Comparison to top-2 gap arguments.}
If one tried to state the same lemma using the top-2 gap $T = l_{p_{\mathrm{sink}}} - \max_{p \ne p_{\mathrm{sink}}} l_p$ instead of the aggregate $G$, one would need an extra $\log N$ term to control the sum $\sum_{p \ne p_{\mathrm{sink}}} e^{l_p}$, giving $1 - a'_{p_{\mathrm{sink}}} \le N \exp(-T + B_{\mathrm{gap}})$ or equivalently $\exp(-T + B_{\mathrm{gap}} + \log N)$, which is vacuous unless $T > \log N + B_{\mathrm{gap}}$.
The aggregate-gap form sidesteps this and matches the measured quantity directly.

\subsection{Frobenius / Stable-Rank Restatement of Theorem~\ref{thm:perturb}}
\label{app:stable_rank}

The bound is stated in terms of $\opnorm{\Delta W_q}$ and $\opnorm{\Delta W_k}$.
Practitioners often report the Frobenius norms $\frob{\Delta W_q}, \frob{\Delta W_k}$ instead, and the conversion is via per-projection stable ranks.

\paragraph{The restatement.}
Let $r_q := \frob{\Delta W_q}^2 / \sigma_{\max}^2(\Delta W_q)$ and analogously $r_k$.
By definition $\opnorm{\Delta W_*} = \sigma_{\max}(\Delta W_*) = \frob{\Delta W_*}/\sqrt{r_*}$.
Keeping the two operator norms separate in the proof before upper-bounding both by a common $\varepsilon$, and substituting $\opnorm{\Delta W_*}=\frob{\Delta W_*}/\sqrt{r_*}$, gives the Frobenius/stable-rank restatement
\begin{equation}
  B_{\mathrm{gap}}
  \;\le\;
  c\,\|\gamma_q\|_\infty\,\|\gamma_k\|_\infty
  \left(
    \frac{\frob{\Delta W_q}}{\sqrt{r_q}\,m_q}
    +
    \frac{\frob{\Delta W_k}}{\sqrt{r_k}\,m_k}
  \right)
  + O(\varepsilon^2),
  \label{eq:Btight_app}
\end{equation}
with empirical inputs $\{r_q, r_k, \frob{\Delta W_q}, \frob{\Delta W_k}, m_q, m_k, \|\gamma_q\|_\infty, \|\gamma_k\|_\infty\}$.
This is a rewriting using stable ranks, not a tightening. If one perturbation is zero, its corresponding term is taken to be zero; otherwise the expression applies with the stable rank defined above.

\paragraph{Role of the restatement.}
The Frobenius form makes the structural observation explicit: the
QK-RMSNORM side has no explicit raw-input-magnitude factor, only the stable ranks $r_q,r_k$ and the $\gamma,m$ constants.
The number that enters the main text is the directly measured $B_{\mathrm{gap}}$ on the calibration set (Section~\ref{sec:method}), and the measured value on Qwen3 ($2.06$ nats) is about $2.4\times$ smaller than on Qwen2.5 ($5.01$ nats) on the same $15$-prompt calibration set (Table~\ref{tab:G_Bgap}).
The Qwen3-versus-Qwen2.5 comparison is consistent with the theorem's structural prediction that the QK-RMSNORM side is less sensitive to raw input magnitude, but we do not claim it as a causal isolation of QK-RMSNORM alone, since Qwen3 and Qwen2.5 differ in pre-training data, RL stage, and tokenizer in addition to QK-RMSNORM.

\subsection{Multi-Layer Composition Remark}
\label{app:multi_layer}

Theorem~\ref{thm:perturb} and Lemma~\ref{lem:saturation} are single-layer statements.
A formal multi-layer composition requires additional control of the value vectors, residual stream, MLP, and post-attention LayerNorm, which we do not provide here.
We instead state an informal chaining intuition.
If at every layer $\ell$ the per-layer base aggregate gap $G_{\mathrm{base}, \ell}$ exceeds the per-layer gap perturbation by a margin $G_{\mathrm{base}, \ell} - B_{\mathrm{gap}, \ell} \ge \mu$, then by Lemma~\ref{lem:saturation} the perturbed non-sink mass at that layer is at most $e^{-\mu}$, and the next layer \emph{is approximately fed} a hidden state close to the LLM baseline under additional Lipschitz assumptions on the intervening operators.
We do not formalize these assumptions and report this remark only as motivation for the per-layer diagnostic.

\paragraph{Caveat.}
In our calibration measurements, the marginal proxy is negative in the top late layers of Qwen2.5 (Section~\ref{sec:method}).
A tight multi-layer theorem (in the spirit of~\citep{olsson2022context}'s circuit composition) is future work.

\subsection{Why Layerwise QK-RMSNorm Does Not Provide the Same Per-Head Guarantee}
\label{app:layerwise_deriv}

This subsection derives the claim of Remark~\ref{rem:layerwise}, that a shared RMS denominator does not provide the same per-head margin guarantee as per-head $\qknorm$.

\paragraph{Per-head normalization.}
Under per-head $\qknorm$, head $h$ is divided by its own RMS, $\mathrm{RMS}(Q_h) = \|Q_h\|/\sqrt{d_h}$, so
\begin{equation}
\|\tilde Q_h\| = \Bigl\| \gamma_q^{(h)} \odot \frac{Q_h}{\mathrm{RMS}(Q_h)} \Bigr\|
\le \|\gamma_q^{(h)}\|_\infty \sqrt{d_h} .
\label{eq:perhead_env}
\end{equation}
The input magnitude $\|Q_h\|$ cancels, which is exactly why the $\qknorm$ branch of Theorem~\ref{thm:perturb} carries no $\|\xi\|$ factor.

\paragraph{Layerwise normalization.}
Under layerwise $\qknorm$, all $H$ heads share one denominator over the concatenated vector $Q_{\mathrm{cat}} = [Q_1; \cdots; Q_H]$, so $\mathrm{RMS}(Q_{\mathrm{cat}}) = \|Q_{\mathrm{cat}}\|/\sqrt{H d_h}$.
Extracting the $h$-th block of $\tilde Q_{\mathrm{cat}} = \gamma_q \odot Q_{\mathrm{cat}}/\mathrm{RMS}(Q_{\mathrm{cat}})$ gives
\begin{equation}
\|\tilde Q_h\| \le \|\gamma_q^{(h)}\|_\infty \|Q_h\| \frac{\sqrt{H d_h}}{\|Q_{\mathrm{cat}}\|} ,
\label{eq:layerwise_env}
\end{equation}
which replaces the per-head envelope of Equation~\ref{eq:perhead_env}.
The same derivation applies to keys.

\paragraph{What the shared denominator costs.}
Write $\rho_h := \|Q_h\| / \|Q_{\mathrm{cat}}\|$ for the share of layer energy carried by head $h$, so that Equation~\ref{eq:layerwise_env} reads $\|\gamma_q^{(h)}\|_\infty \rho_h \sqrt{H d_h}$.
Both normalizations remove the input magnitude, since $\rho_h$ is unchanged when $\xi_q$ is rescaled.
What per-head normalization additionally provides is a scale
independent of the head energy, namely $\sqrt{d_h}$ up to the
learned factor $\|\gamma_q^{(h)}\|_\infty$.
The layerwise scale instead additionally tracks $\rho_h$.
Under equal per-head energy $\rho_h = 1/\sqrt{H}$ and Equation~\ref{eq:layerwise_env} collapses to Equation~\ref{eq:perhead_env}, but a head carrying less than its share has $\rho_h < 1/\sqrt{H}$ and is compressed relative to its neighbours.

\paragraph{Consequence for the margin.}
Compression attenuates that head's attention logits, and with them the sink's logit lead, so $G_{\mathrm{base}}$ can be small at heads that carry little of the layer's energy.
The margin $G_{\mathrm{base}} - B_{\mathrm{gap}}$ can then fail from the base side rather than through a larger perturbation, which is the opposite of the no-$\qknorm$ failure mode where $\|\xi\|$ enters the bound on $B_{\mathrm{gap}}$ directly.
Aggregate sink mass is a sum over heads and can stay high while individual compressed heads carry almost no lead, so preservation of the aggregate does not imply preservation of the per-head margin.
Appendix~\ref{app:molmo} reports the empirical counterpart, where Molmo2-O has the lowest per-head $G_{\mathrm{base}}$ among the calibrated pairs at an unremarkable $B_{\mathrm{gap}}$, while the no-$\qknorm$ Qwen2.5-VL shows the reverse profile.

\subsection{Numerical Sanity Checks}
\label{app:numerical}

We verify two pieces on small synthetic cases.

\paragraph{Theorem~\ref{thm:perturb}, $d = 16$, $d_h = 4$, no $\qknorm$.}
We consider a RoPE-free synthetic case with Gaussian-initialized $W_q, W_k$.
We instantiate a sequence of $N = 8$ key positions with i.i.d.\ unit-norm Gaussian backgrounds, then form $\xi_{p_{\mathrm{sink}}}$ by adding a massive-activation-style channel at index $3$ with value $10$ on top of its background, while $\xi_p$ for $p \ne p_{\mathrm{sink}}$ retain only the unit-norm background.
The sink hidden state thus has a sink/non-sink magnitude gap dominated by the size-$10$ channel rather than being constrained to unit norm.
We draw Gaussian $\Delta W_*$ with $\opnorm{\Delta W_*} = 0.1$.
The measured sink-anchored quantity is
$\sup_{p \ne p_{\mathrm{sink}}}|\delta_p - \delta_{p_{\mathrm{sink}}}| = 0.169$.
Because this synthetic case omits RoPE, the proof admits the sharper specialization with
$\|\xi_{p_{\mathrm{sink}}} - \xi_p\|$
before the triangle-inequality relaxation used in Equation~\ref{eq:thm_perturb_no}.
This specialized bound is $0.197$, giving a ratio of $0.86$.
Since
$\|\xi_{p_{\mathrm{sink}}} - \xi_p\|
\le
\|\xi_{p_{\mathrm{sink}}}\| + \|\xi_p\|$,
the more general RoPE-aware bound in Equation~\ref{eq:thm_perturb_no} is also satisfied.

\paragraph{Lemma~\ref{lem:saturation}, $G = 3.0$, $B_{\mathrm{gap}} = 0.4$, $N = 64$.}
Direct softmax computation with $l_{p_{\mathrm{sink}}} = 5$ and the $N - 1$ non-sink positions equally distributed to give $G = 3.0$.
To realize the worst-case gap perturbation, we set $\delta_p = +B_{\mathrm{gap}}/2$ for every $p \ne p_{\mathrm{sink}}$ and $\delta_{p_{\mathrm{sink}}} = -B_{\mathrm{gap}}/2$, so that $\sup_{p \ne p_{\mathrm{sink}}}|\delta_p - \delta_{p_{\mathrm{sink}}}| = B_{\mathrm{gap}}$ is attained.
Empirical $1 - a'_{p_{\mathrm{sink}}} = 0.069$.
The logistic bound from Equation~\ref{eq:lem_sat} is $\sigma(0.4 - 3.0) = \sigma(-2.6) \approx 0.069$.
Ratio $1.00$ (the bound is tight under the worst-case perturbation).

\twocolumn

\section{Evaluation Setup}
\label{app:setup}

This appendix collects evaluation infrastructure used across the paper: the harness protocols (Appendix~\ref{sec:protocol}), the backbone-extraction pipeline (Appendix~\ref{sec:extraction}), and hardware/software details for reproducibility (Appendix~\ref{sec:impl}).

\subsection{Evaluation Protocols}
\label{sec:protocol}

We evaluate all text-side capability with \texttt{lm-evaluation-harness} v0.4.12~\citep{eval-harness} under two protocols whose disagreement on a single backbone is itself diagnostic.

\paragraph{Instruct protocol.}
\texttt{--apply\_chat\_template}, \texttt{--num\_fewshot 0}, \texttt{--batch\_size 8}, bf16.
This matches the intended chat-style usage of instruction-tuned models.

\paragraph{Community-default protocol.}
The harness's per-task default few-shot count ($\texttt{mmlu}{=}5$, $\texttt{gsm8k}{=}8$, etc.) with the chat template \emph{disabled}.
This provides a common leaderboard-style comparison.

\paragraph{Nine-task suite.}
MMLU, MMLU-Pro, BoolQ, MBPP (code, pass@1), RULER (long-context), GSM8K-CoT, IFEval (\texttt{prompt\_level\_strict\_acc}), GPQA-Diamond-CoT, and EQ-Bench (raw points).
IFEval is the headline metric for sink corruption because its checker grades surface format directly~\citep{zhou2023instruction}.

\paragraph{Pipeline sanity.}
Our LLM-side numbers reproduce the Qwen blog values on Qwen2.5-7B-Instruct GSM8K-CoT to within $0.93$ pt ($90.67$ vs.\ $91.6$) and MMLU-Pro to within $0.46$ pt ($55.84$ vs.\ $56.30$) when run under matched OpenCompass configurations.
The deltas reported throughout this paper compare a reference LLM to its VLM-LM under the same protocol, reducing sensitivity to framework-specific evaluation differences.

\subsection{Backbone Extraction}
\label{sec:extraction}

To evaluate text-only behavior of a VLM with \texttt{lm-evaluation-harness}, we extract its language backbone as a standalone causal LM.
For Qwen2.5-VL and the LLaVA-OneVision / LLaVA-Llama3 families, we remap the language-backbone parameters to the corresponding \texttt{<Family>ForCausalLM} namespace while preserving the VLM's original input embeddings and output language-model head, including the original \texttt{lm\_head.weight}.
For Qwen3-VL, the same procedure additionally preserves the
\texttt{q\_norm} / \texttt{k\_norm} scale parameters used in
per-head normalization.
For our QK-RMSNorm-injected Qwen2.5-3B variant (Section~\ref{sec:controls}), we use a self-contained \texttt{trust\_remote\_code} module so that the extracted backbone follows the same numerical pathway as during training.
All extracted backbones are saved to \texttt{safetensors}, and the round-trip reproduces the original language module's logits on a held-out text-only probe.

\subsection{Implementation Details and Reproducibility}
\label{sec:impl}

\paragraph{Hardware.}
Model training (the Section~\ref{sec:negative} $\qknorm$ injection experiment) ran on four NVIDIA H100 GPUs (80 GB), with two GPUs allocated to each variant.
All text-side evaluation (lm-evaluation-harness across the nine-task suite for all 6 pairs and the text-only endpoint control, with MMLU-Pro and RULER measured on the five headline pairs and Molmo2-O excluded for those two), the Sink Strength calibration forward passes (Section~\ref{sec:predictor}, $\sim 8$~s per 7B backbone), and the $G_{\mathrm{base}}, B_{\mathrm{gap}}$ measurement on the $15$-prompt calibration set ($\sim 52$~s/pair, full pipeline including VLM-LM forward and $m_q/m_k$ compute) ran on a single NVIDIA A6000 (48 GB) in bf16.

\paragraph{Software.}
PyTorch 2.10+cu128, transformers 4.57.6, DeepSpeed 0.19.0, lm-evaluation-harness 0.4.12, VLMEvalKit.

\subsection{Models and Datasets}
\label{app:datasets}

This subsection collects Hugging Face URLs and paper references for every model and dataset used in the paper.

\paragraph{Vision-language models (headline).}
The five headline VLMs and the Molmo2-O control (Section~\ref{sec:problem_setup}):
\begin{itemize}[leftmargin=*,itemsep=1pt,topsep=2pt]
  \item Qwen3-VL-8B-Instruct~\citep{bai2025qwen3}: \url{https://huggingface.co/Qwen/Qwen3-VL-8B-Instruct}
  \item InternVL3.5-8B~\citep{wang2025internvl3}: \url{https://huggingface.co/OpenGVLab/InternVL3_5-8B}
  \item Qwen2.5-VL-7B-Instruct~\citep{bai2025qwen25vltechnicalreport}: \url{https://huggingface.co/Qwen/Qwen2.5-VL-7B-Instruct}
  \item InternVL3-8B~\citep{zhu2025internvl3}: \url{https://huggingface.co/OpenGVLab/InternVL3-8B}
  \item LLaVA-OneVision-Qwen2-7B-OV~\citep{li2024llava}: \url{https://huggingface.co/lmms-lab/llava-onevision-qwen2-7b-ov}
  \item Molmo2-O-7B~\citep{clark2026molmo2}: \url{https://huggingface.co/allenai/Molmo2-O-7B}
\end{itemize}

\paragraph{Vision-language models (extended 17-pair panel).}
The eleven additional VLMs in the IFEval-specific 17-pair panel (Appendix~\ref{app:predictor_17pair}):
\begin{itemize}[leftmargin=*,itemsep=1pt,topsep=2pt]
  \item Qwen3-VL-4B-Instruct~\citep{bai2025qwen3}: \url{https://huggingface.co/Qwen/Qwen3-VL-4B-Instruct}
  \item Qwen3-VL-2B-Instruct~\citep{bai2025qwen3}: \url{https://huggingface.co/Qwen/Qwen3-VL-2B-Instruct}
  \item Qwen3-VL-30B-A3B-Instruct~\citep{bai2025qwen3}: \url{https://huggingface.co/Qwen/Qwen3-VL-30B-A3B-Instruct}
  \item Ovis2-8B~\citep{lu2024ovis}: \url{https://huggingface.co/AIDC-AI/Ovis2-8B}
  \item Ovis2-34B~\citep{lu2024ovis}: \url{https://huggingface.co/AIDC-AI/Ovis2-34B}
  \item MolmoE-1B-0924~\citep{deitke2025molmo}: \url{https://huggingface.co/allenai/MolmoE-1B-0924}
  \item InternVL3-2B~\citep{zhu2025internvl3}: \url{https://huggingface.co/OpenGVLab/InternVL3-2B}
  \item InternVL3-14B~\citep{zhu2025internvl3}: \url{https://huggingface.co/OpenGVLab/InternVL3-14B}
  \item MiniCPM-V-4.5~\citep{yu2025minicpmv45}: \url{https://huggingface.co/openbmb/MiniCPM-V-4_5}
  \item LLaVA-NeXT-Mistral-7B~\citep{liu2024improved}: \url{https://huggingface.co/llava-hf/llava-v1.6-mistral-7b-hf}
  \item Idefics3-8B-Llama3~\citep{laurenccon2024building}: \url{https://huggingface.co/HuggingFaceM4/Idefics3-8B-Llama3}
\end{itemize}

\paragraph{Text-only reference LLMs.}
The text-only reference models evaluated alongside the extracted VLM-LMs:
\begin{itemize}[leftmargin=*,itemsep=1pt,topsep=2pt]
  \item Qwen3-8B~\citep{yang2025qwen3}: \url{https://huggingface.co/Qwen/Qwen3-8B}
  \item Qwen3-4B~\citep{yang2025qwen3}: \url{https://huggingface.co/Qwen/Qwen3-4B}
  \item Qwen3-1.7B~\citep{yang2025qwen3}: \url{https://huggingface.co/Qwen/Qwen3-1.7B}
  \item Qwen3-30B-A3B-Instruct-2507~\citep{yang2025qwen3}: \url{https://huggingface.co/Qwen/Qwen3-30B-A3B-Instruct-2507}
  \item Qwen2.5-7B-Instruct~\citep{qwen2025qwen25technicalreport}: \url{https://huggingface.co/Qwen/Qwen2.5-7B-Instruct}
  \item Qwen2.5-1.5B-Instruct~\citep{qwen2025qwen25technicalreport}: \url{https://huggingface.co/Qwen/Qwen2.5-1.5B-Instruct}
  \item Qwen2.5-14B-Instruct~\citep{qwen2025qwen25technicalreport}: \url{https://huggingface.co/Qwen/Qwen2.5-14B-Instruct}
  \item Qwen2.5-32B-Instruct~\citep{qwen2025qwen25technicalreport}: \url{https://huggingface.co/Qwen/Qwen2.5-32B-Instruct}
  \item Qwen2-7B-Instruct~\citep{yang2024qwen2technicalreport}: \url{https://huggingface.co/Qwen/Qwen2-7B-Instruct}
  \item Qwen2.5-3B-Instruct~\citep{qwen2025qwen25technicalreport}: \url{https://huggingface.co/Qwen/Qwen2.5-3B-Instruct}
  \item Olmo-3-7B-Instruct~\citep{olmo2025olmo}: \url{https://huggingface.co/allenai/Olmo-3-7B-Instruct}
  \item OLMoE-1B-7B-0924-Instruct~\citep{muennighoff2025olmoe}: \url{https://huggingface.co/allenai/OLMoE-1B-7B-0924-Instruct}
  \item Mistral-7B-Instruct-v0.2~\citep{jiang2023mistral7b}: \url{https://huggingface.co/mistralai/Mistral-7B-Instruct-v0.2}
  \item Llama-3.1-8B-Instruct~\citep{grattafiori2024llama}: \url{https://huggingface.co/meta-llama/Llama-3.1-8B-Instruct}
\end{itemize}

\paragraph{Vision tower and projector donor.}
The injection experiment of Appendix~\ref{app:c3details} takes its vision tower and projector from Qwen2.5-VL-3B-Instruct~\citep{bai2025qwen25vltechnicalreport}, \url{https://huggingface.co/Qwen/Qwen2.5-VL-3B-Instruct}.

\paragraph{Modality-control variant.}
The text-only further-training endpoint used in the 7B modality comparison
(Appendix~\ref{app:lm_sft}): Qwen2.5-7B-Instruct-1M, \url{https://huggingface.co/Qwen/Qwen2.5-7B-Instruct-1M}.

\paragraph{Text evaluation datasets.}
Four format-sensitive tasks and five capability tasks reported in the paper, all run via \texttt{lm-evaluation-harness}~\citep{eval-harness} (Appendix~\ref{sec:protocol}):
\begin{itemize}[leftmargin=*,itemsep=1pt,topsep=2pt]
  \item IFEval~\citep{zhou2023instruction}: \url{https://huggingface.co/datasets/google/IFEval}
  \item EQ-Bench v2.1~\citep{paech2023eq}: \url{https://github.com/EQ-bench/EQ-Bench}
  \item GSM8K~\citep{cobbe2021training}: \url{https://huggingface.co/datasets/openai/gsm8k}
  \item GPQA-Diamond~\citep{rein2023gpqa}: \url{https://huggingface.co/datasets/Idavidrein/gpqa}
  \item MMLU~\citep{hendrycks2020measuring}: \url{https://huggingface.co/datasets/cais/mmlu}
  \item BoolQ~\citep{clark2019boolq}: \url{https://huggingface.co/datasets/google/boolq}
  \item MMLU-Pro~\citep{wang2024mmlu}: \url{https://huggingface.co/datasets/TIGER-Lab/MMLU-Pro}
  \item MBPP~\citep{austin2021program}: \url{https://huggingface.co/datasets/google-research-datasets/mbpp}
  \item RULER~\citep{hsieh2024ruler}: 
  \url{https://github.com/NVIDIA/RULER}
\end{itemize}

\paragraph{Calibration prompts.}
The $15$-prompt set used for Sink Strength (Section~\ref{sec:predictor}) and the $G_{\mathrm{base}}, B_{\mathrm{gap}}$ measurement (Section~\ref{sec:method}) consists of short instruction-following prompts we wrote in the style of IFEval. No IFEval instance is used, so the calibration set is disjoint from the data we evaluate on by construction.

\paragraph{Licenses and terms of use.}
All Hugging Face checkpoints and evaluation datasets listed above are used under their published licenses, predominantly Apache~2.0 and MIT, with a small number under model-specific licenses (the Llama~3.1 Community License for Llama-3.1-8B-Instruct, and the research-only Qwen Research License for Qwen2.5-3B-Instruct and Qwen2.5-VL-3B-Instruct, which we use for research purposes only).
We download each artifact under its stated terms and do not redistribute the underlying weights or data.
The evaluation harness \texttt{lm-evaluation-harness}~\citep{eval-harness} is MIT-licensed.

\section{Per-Task and Behavioral Evidence}
\label{app:behavior}

This appendix expands the per-task behavioral split surfaced in Section~\ref{sec:problem_setup}: capability vs.\ format-sensitive coverage (Appendix~\ref{app:knowledge_coverage}), IFEval failure-category counts (Appendix~\ref{app:failure_categories}), an outcome-bucketed sink probe (Appendix~\ref{app:sink_outcome}), and an extended capability panel on coding, reasoning, and long-context retrieval (Appendix~\ref{app:extended_bench}).

\subsection{Capability Preservation vs.\ Format-Sensitive Degradation}
\label{app:knowledge_coverage}

Section~\ref{sec:problem_setup} shows that several non-format-sensitive capabilities are comparatively preserved while format-following degrades.
We verify this here on the knowledge- and comprehension-oriented
benchmarks surfaced in the main text, and Appendix~\ref{app:extended_bench} extends the check to coding, long-context retrieval, and advanced reasoning.

\paragraph{Knowledge and comprehension.}
Table~\ref{tab:knowledge_coverage} reports the per-pair $\Delta$ on MMLU and BoolQ alongside the four format-sensitive anchors under the same instruct protocol.
MMLU and BoolQ drops never exceed $2.3$~pt across the six pairs, and on three pairs they even improve.
On the same backbones the headline IFEval gaps reach $-9.6$ pt (no-$\qknorm$) and $-18.7$ pt (layerwise).

\begin{table}[ht]
\centering
\setlength{\tabcolsep}{2pt}
\fontsize{8}{10}\selectfont
\begin{tabular}{@{}l r r >{\columncolor{gray!15}}r >{\columncolor{gray!15}}r >{\columncolor{gray!15}}r >{\columncolor{gray!15}}r@{}}
\toprule
 & \multicolumn{2}{c}{Knowledge} & \multicolumn{4}{c}{Format-sensitive} \\
\cmidrule(lr){2-3} \cmidrule(lr){4-7}
Pair & MMLU & BoolQ & IFEval & EQ & GSM8K & GPQA \\
\midrule
Qwen3-VL     & $+1.6$ & $+6.0$ & $-5.4$  & $-2.4$  & $-2.1$  & $-1.9$ \\
InternVL3.5  & $+3.0$ & $+7.6$ & $-1.8$  & $-0.8$  & $-1.5$  & $-0.6$ \\
\addlinespace[1pt]
Qwen2.5-VL   & $-2.1$ & $-1.2$ & $-9.6$  & $-11.2$ & $-14.3$ & $-9.8$ \\
InternVL3    & $+3.0$ & $+4.5$ & $-8.9$  & $-10.3$ & $-12.5$ & $-7.1$ \\
LLaVA-OV     & $-2.3$ & $+0.1$ & $-7.9$  & $-8.6$  & $-11.9$ & $-8.4$ \\
\midrule
Molmo2-O     & $-0.0$ & $-1.5$ & $-18.7$ & $-64.6$ & $-42.7$ & $-19.7$ \\
\midrule
\textbf{Mean} & $\mathbf{+0.5}$ & $\mathbf{+2.6}$ & $\mathbf{-8.7}$ & $\mathbf{-16.3}$ & $\mathbf{-14.2}$ & $\mathbf{-7.9}$ \\
\bottomrule
\end{tabular}
\caption{Per-pair $\Delta$ on knowledge- and comprehension-oriented benchmarks (MMLU, BoolQ) vs.\ format-sensitive benchmarks (shaded). Same instruct protocol throughout, except that the Molmo2-O GPQA baseline is the published Olmo-3-7B-Instruct score rather than our own matched run, which would give $-10.6$. Pair-level QK category in Table~\ref{tab:pair_metadata}. EQ is EQ-Bench v2.1 (raw points); GSM8K is GSM8K-CoT flexible-extract; GPQA is GPQA-Diamond-CoT flexible-extract.}
\label{tab:knowledge_coverage}
\end{table}

\paragraph{Ruling out a format confound.}
A sceptical reader could attribute the contrast to MCQA-vs-open-generation format rather than capability: MMLU and BoolQ preserve because they are loglikelihood-graded letter-picks, and the format-sensitive tasks degrade because they are open generation.
Two facts in our data rule out this simple reading.
First, GPQA-Diamond-CoT is MCQA in format (the answer is one of A/B/C/D, scored by flexible-extract on a free-text rationale), yet it drops by $-7.1$ to $-9.8$~pt on the three no-$\qknorm$ pairs, almost the full IFEval-scale magnitude.
Second, GSM8K-CoT is open generation with a multi-step reasoning trajectory, yet it stays within $\pm 2.1$~pt on the two per-head $\qknorm$ pairs.
Preservation therefore tracks the capability being probed
(knowledge/comprehension vs.\ instruction-following / reasoning-trajectory) and the QK protection of the reference LLM, not the task's MCQA-vs-generation format.

\subsection{Failure-Category Breakdown}
\label{app:failure_categories}

For each VLM-LM we partition the IFEval set into the four buckets reported in Table~\ref{tab:regression} (\emph{both\_pass}, \emph{regression}, \emph{recovered}, \emph{both\_fail}).
The \emph{regression} bucket (prompts the LLM passes but the VLM-LM fails) is the locus of the per-pair degradation.
Table~\ref{tab:failure_categories} reports the dominant failing IFEval instruction categories per pair (counts of failed instructions within the regression bucket).
"Punctuation" is largely \texttt{no\_comma}, "combination" is largely \texttt{repeat\_prompt} or \texttt{two\_responses}, and "length\_constraints" is largely \texttt{number\_words} or \texttt{number\_paragraphs}.
"Language" is the \texttt{response\_language} constraint, "startend" covers \texttt{quotation} and \texttt{end\_checker}, and "detectable\_format" covers \texttt{multiple\_sections} and \texttt{number\_bullet\_lists}.
"Change\_case" covers \texttt{english\_capital} and \texttt{english\_lowercase}, and "keywords" covers \texttt{existence} and \texttt{letter\_frequency}.

\begin{table}[ht]
\centering
\setlength{\tabcolsep}{1pt}
\fontsize{8}{10}\selectfont
\begin{tabular}{@{}l c c c c c c@{}}
\toprule
 & \rotatebox{60}{Qwen3-VL}
 & \rotatebox{60}{InternVL3.5}
 & \rotatebox{60}{Qwen2.5-VL}
 & \rotatebox{60}{InternVL3}
 & \rotatebox{60}{LLaVA-OV}
 & \rotatebox{60}{Molmo2-O} \\
\midrule
length\_constraints  &  9 &  8 & 17 & 16 &  7 & \textbf{20} \\
detectable\_format   &  8 &  5 & 10 &  8 & 20 & \textbf{30} \\
keywords             &  9 &  6 & 13 & 18 &  4 & \textbf{19} \\
change\_case         &  5 &  8 & 11 & \textbf{17} & 16 &  8 \\
combination          & 14 &  - &  6 & 13 &  7 & \textbf{22} \\
startend             &  7 &  4 &  4 &  7 & 12 & \textbf{18} \\
language             &  5 &  - &  7 &  5 &  8 & \textbf{10} \\
punctuation          &  3 &  1 & \textbf{28} & 21 &  8 &  3 \\
detectable\_content  &  3 &  3 &  3 &  5 &  1 & \textbf{11} \\
\midrule
\textbf{Total}       & \textbf{56} & \textbf{43} & \textbf{91} & \textbf{96} & \textbf{83} & \textbf{126} \\
\bottomrule
\end{tabular}
\caption{Regressed IFEval prompts, counted per violated instruction category, per pair. A prompt that violates several categories is counted in each, and a prompt whose violations fall outside the listed categories is counted in none, so the columns need not sum to the Total row. The Total row is the regressed-prompt count and matches the \emph{Lost} column of Table~\ref{tab:regression}. "$-$" marks absence; \textbf{bold} marks the row-max.}
\label{tab:failure_categories}
\end{table}

\subsection{Outcome-Bucketed Sink Probe}
\label{app:sink_outcome}

For each VLM-LM we partition the IFEval prompts by \texttt{prompt\_level\_strict\_acc} into a pass bucket ($\ge 0.5$) and a fail bucket ($< 0.5$), then run a single forward pass per prompt with \texttt{output\_attentions=True} and aggregate position-0 attention probability over (late layer $\times$ head $\times$ last-$8$ query positions).
We report the per-prompt mean (so the bootstrap is over prompts, not flattened head/layer observations) and a $95\%$ percentile CI from bootstrap over $100$ prompts sampled per outcome bucket ($200$ per pair).

\subsection{Extended Capability Panel}
\label{app:extended_bench}

Beyond the knowledge/comprehension benchmarks and the four format-sensitive tasks, the nine-task suite of Appendix~\ref{sec:protocol} covers coding (MBPP, execution pass@1), advanced multi-domain reasoning (MMLU-Pro), and long-context retrieval (RULER).
Table~\ref{tab:extended_bench} reports these three on the five headline pairs; Molmo2-O, the layerwise control, is not included in this panel.

The picture across the three is mixed rather than uniform.
Coding rises (mean $+7.3$), long-context retrieval is roughly flat (mean $-1.1$), and advanced reasoning declines mildly (mean $-2.7$, worst $-7.4$ on Qwen2.5-VL).
What separates them from the format-sensitive cluster is the comparison on the same five pairs, where the four format-sensitive means run $-5.6$ to $-8.5$~pt.
None of these three benchmarks separates by per-head $\qknorm$, unlike the format-sensitive four.
The loss therefore tracks the demand for controlled output
formatting rather than the task's domain or difficulty alone,
which is what the mechanism of Section~\ref{sec:method} predicts.

\begin{table}[ht]
\centering
\setlength{\tabcolsep}{6pt}
\fontsize{8}{10}\selectfont
\begin{tabular}{@{}l l r r r@{}}
\toprule
Pair & QK & MBPP & MMLU-Pro & RULER \\
\midrule
Qwen3-VL     & per-head & $+7.6$  & $-3.4$ & $-0.2$ \\
InternVL3.5  & per-head & $-2.4$  & $+0.6$ & $-0.4$ \\
\addlinespace[1pt]
Qwen2.5-VL   & none     & $+14.2$ & $-7.4$ & $+0.3$ \\
InternVL3    & none     & $+21.4$ & $-0.2$ & $+0.6$ \\
LLaVA-OV     & none     & $-4.4$  & $-3.2$ & $-6.0$ \\
\midrule
\textbf{Mean} & & $\mathbf{+7.3}$ & $\mathbf{-2.7}$ & $\mathbf{-1.1}$ \\
\bottomrule
\end{tabular}
\caption{Extended capability panel: per-pair $\Delta$ (VLM-LM $-$ reference LLM, pp) on coding (MBPP, execution pass@1), advanced multi-domain reasoning (MMLU-Pro), and long-context retrieval (RULER), across the five headline pairs; Molmo2-O excluded. Same instruct protocol; per-pair QK category as in Table~\ref{tab:pair_metadata}.}
\label{tab:extended_bench}
\end{table}

\FloatBarrier

\section{Mechanism Details}
\label{app:mechanism}

This appendix expands the mechanism analyses of Section~\ref{sec:method}: the calibration-set size choice (Appendix~\ref{app:calib_sweep}), the two layerwise controls (Appendix~\ref{app:molmo}), channel-level ablation (Appendix~\ref{app:c1details}), and a random-direction perturbation control (Appendix~\ref{app:e2details}).

\subsection{Calibration Sample-Size Robustness}
\label{app:calib_sweep}

The $(G_{\mathrm{base}}, B_{\mathrm{gap}})$ measurement of Section~\ref{sec:method} and the Sink Strength $S$ of Section~\ref{sec:predictor} both use a default of $N\!=\!15$ calibration prompts.
We verify this choice is converged by sweeping $N \in \{5, 10, 15, 25, 50, 100\}$ with $3$ random seeds per cell, drawing each subset uniformly from a $100$-prompt calibration pool disjoint from the IFEval test set.
At $N\!=\!100$ every seed draws the entire pool, so the seed scatter vanishes there by construction.
This produces $18$ measurements per pair across the $6$ pairs, $108$ filled cells in total.
The hyperparameter is measurement-time only: it never trains a model and never sees benchmark labels, so it controls only the variance of the estimator.

\begin{figure*}[t]
\centering
\includegraphics[width=\textwidth]{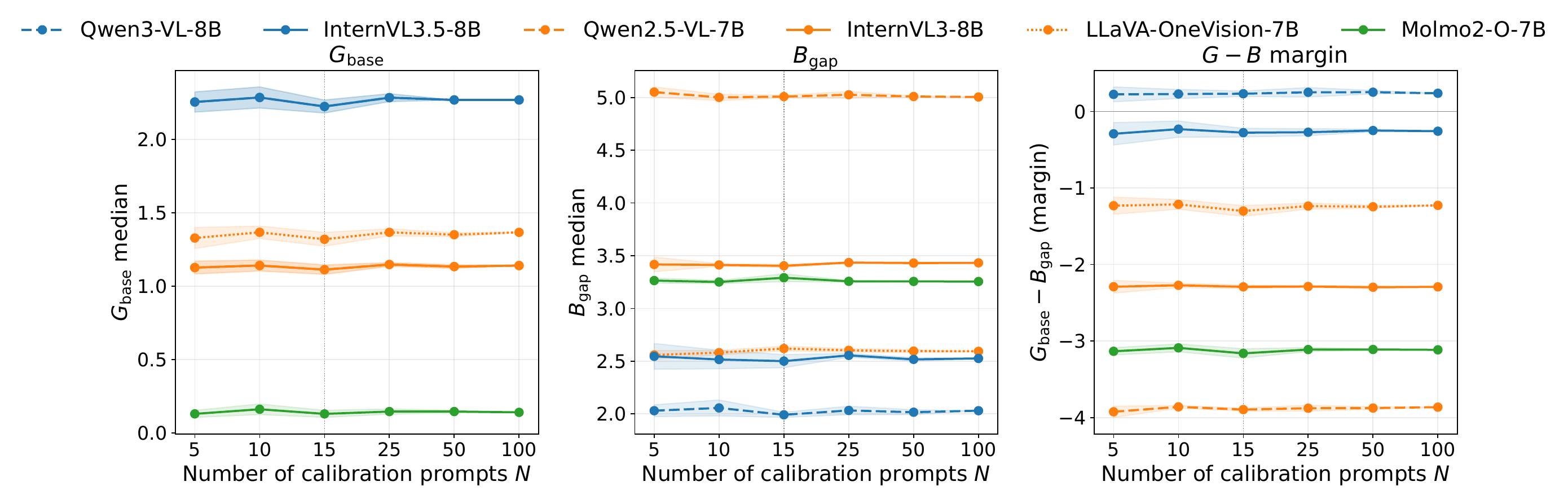}
\caption{$G_{\mathrm{base}}$, $B_{\mathrm{gap}}$, and the $G_{\mathrm{base}}-B_{\mathrm{gap}}$ margin versus calibration size $N$ for each of the six pairs. Lines are seed means, shaded bands are $\pm \sigma$ across $3$ seeds. The vertical dashed line marks the paper-default $N\!=\!15$; the horizontal line in the right panel marks the $G-B = 0$ sign-flip boundary.}
\label{fig:calib_sweep}
\end{figure*}

\paragraph{Observations.}
(i) \emph{Plateau from $N \approx 10$.}
Seed-mean shifts are $\le 0.05$ in $G_{\mathrm{base}}$, $\le 0.05$ in $B_{\mathrm{gap}}$, and $\le 0.06$ in $G-B$ between $N\!=\!10$ and $N\!=\!100$ for every pair; by $N\!=\!15$ the curves are indistinguishable from the $N\!=\!100$ asymptote within seed scatter.
(ii) \emph{Seed scatter contracts as expected.}
$\sigma(G-B)$ drops from $\sim\!0.15$ at $N\!=\!5$ (worst case, InternVL3.5) to $\le 0.07$ by $N\!=\!15$ and $\le 0.02$ by $N\!=\!50$.
(iii) \emph{Pair ordering invariant.}
Across all $108$ (pair, $N$, seed) cells, the sign of $G-B$ is preserved (positive only for Qwen3-VL, negative for the other five), and the rank ordering of pairs on $G-B$ does not flip at any cell.
The behavioral split and the magnitude ordering used in Section~\ref{sec:method} therefore do not depend on the calibration set choice.

\paragraph{Sink Strength $S$.}
The sweep above tracks the margin $G_{\mathrm{base}} - B_{\mathrm{gap}}$.
The predictor of Section~\ref{sec:predictor} uses $S$ alone, which is less stable under small $N$, so we report it separately.
Moving from $N\!=\!15$ to $N\!=\!5$ shifts $S$ on the six pairs by at most $0.11$ (mean $0.09$), against the $\le 0.06$ margin shift above.
The predictor is nevertheless unchanged, with $\rho(S, \Delta_{\mathrm{IFEval}}) = 0.971$ at both sizes and no reordering of the six pairs.
Since $N\!=\!5$ also runs about $3\times$ faster ($\sim\!3$ s versus $\sim\!8$ s per $7$B backbone on an A6000), it reproduces the headline predictor at a third of the cost.
We keep $N\!=\!15$ as the default because the extra prompts are inexpensive and buy a margin against seed variance on broader panels, where the smaller set is not a safe drop-in replacement.

\begin{table*}[ht]
\centering
\setlength{\tabcolsep}{6pt}
\fontsize{8}{10}\selectfont
\begin{tabular}{@{}l r r r r r r@{}}
\toprule
Pair & $N\!=\!5$ & $N\!=\!10$ & $N\!=\!\mathbf{15}$ & $N\!=\!25$ & $N\!=\!50$ & $N\!=\!100$ \\
\midrule
Qwen2.5-VL-7B       & $-3.92 \pm 0.07$ & $-3.86 \pm 0.02$ & $\mathbf{-3.90 \pm 0.02}$ & $-3.88 \pm 0.04$ & $-3.88 \pm 0.02$ & $-3.86 \pm 0.00$ \\
InternVL3-8B        & $-2.29 \pm 0.08$ & $-2.27 \pm 0.03$ & $\mathbf{-2.29 \pm 0.03}$ & $-2.29 \pm 0.01$ & $-2.30 \pm 0.02$ & $-2.29 \pm 0.00$ \\
LLaVA-OneVision-7B  & $-1.23 \pm 0.11$ & $-1.21 \pm 0.06$ & $\mathbf{-1.30 \pm 0.07}$ & $-1.24 \pm 0.04$ & $-1.25 \pm 0.02$ & $-1.23 \pm 0.00$ \\
Qwen3-VL-8B         & $+0.22 \pm 0.10$ & $+0.23 \pm 0.06$ & $\mathbf{+0.23 \pm 0.04}$ & $+0.25 \pm 0.06$ & $+0.25 \pm 0.02$ & $+0.24 \pm 0.00$ \\
InternVL3.5-8B      & $-0.29 \pm 0.15$ & $-0.23 \pm 0.11$ & $\mathbf{-0.28 \pm 0.06}$ & $-0.27 \pm 0.05$ & $-0.25 \pm 0.02$ & $-0.26 \pm 0.00$ \\
Molmo2-O-7B         & $-3.13 \pm 0.05$ & $-3.09 \pm 0.05$ & $\mathbf{-3.16 \pm 0.06}$ & $-3.11 \pm 0.03$ & $-3.11 \pm 0.01$ & $-3.11 \pm 0.00$ \\
\bottomrule
\end{tabular}
\caption{$G_{\mathrm{base}}-B_{\mathrm{gap}}$ margin ($\mu \pm \sigma$ across $3$ seeds) over the full $N$ grid; the paper-default $N\!=\!15$ column is bold. Each cell averages $3$ random $N$-prompt subsets drawn from the $100$-prompt sweep pool, which is separate from the fixed $15$-prompt calibration set used in Table~\ref{tab:G_Bgap} and throughout the paper. This is why the $N\!=\!15$ column differs slightly from that table. The pair ordering is identical in both.}
\label{tab:calib_sweep}
\end{table*}

\FloatBarrier

\subsection{Layerwise Backbones: Aggregate vs.\ Per-Head Sink}
\label{app:molmo}

Molmo2-O-7B is built on Olmo-3, which applies $\qknorm$ \emph{layerwise}: at each layer, RMS is computed across the full $H \cdot d_h$-dimensional concatenated Q (and K) vector before $\gamma_q, \gamma_k$ are applied, rather than across each head's $d_h$-dimensional slice independently (as in Qwen3 / InternVL3.5).
At the model-mean level the layerwise scale preserves a strong sink: position-$0$ attention is $0.250$ ($95\%$ CI $[0.242, 0.258]$), over $8\times$ larger than on the no-$\qknorm$ Qwen2.5-VL backbone ($0.030$).
At the per-head level the distribution is different.
On the same calibration set ($15$-prompt, top-$\sim\!10$ late layers), median $G_{\mathrm{base}}$ across (layer, head, query position) is $+0.24$ nats (the lowest among the six pairs), strong-sink prevalence ($G_{\mathrm{base}}>2$, equivalently $a_{\mathrm{sink}}>0.881$) over non-degenerate late-layer (prompt, layer, head, query) observations is 8.5\%, against 58\% on the per-head $\qknorm$ backbone (Qwen3-8B).
The per-head margin protection of Lemma~\ref{lem:saturation} is therefore not delivered head-by-head, and IFEval drops by $18.7$ pt, GSM8K-CoT by $42.7$ pt, EQ-Bench by $64.6$ pt, GPQA-Diamond-CoT by 19.7 pt against the Olmo-3-7B-Instruct reference (the GPQA figure is measured against the published Olmo-3 score; our own matched run gives $10.6$ pt).
The reading is that aggregate model-level sink mass is not sufficient and that the evidence points to per-head sink strength as the relevant quantity.
We therefore exclude Molmo2-O from the headline 5-pair comparison and report it separately as a distinguishing control.

\paragraph{A second layerwise pair.}
MolmoE-1B, built on OLMoE-1B-7B-0924 and scored here against the released OLMoE-1B-7B-0924-Instruct, applies $\qknorm$ across all heads at once on an MoE backbone and is a second natural layerwise case.
It reaches $S = +1.34$ with $G_{\mathrm{base}} - B_{\mathrm{gap}} = -2.84$ and an IFEval drop of $9.4$ pt, so it follows the same pattern as Molmo2-O ($S = +0.24$, $G_{\mathrm{base}} - B_{\mathrm{gap}} = -3.01$, $18.7$ pt) with the lower-$S$ pair taking the larger drop.
Both layerwise pairs sit well below the per-head pairs, which share $S = +2.36$, and both carry strongly negative margins while differing in base family and in dense versus MoE architecture.
The two are not independent of recipe, since both come from the Molmo and PixMo line, so the recipe-independent derivation of Appendix~\ref{app:layerwise_deriv} is what carries the argument and these two pairs stand as convergent empirical support rather than as a controlled comparison.
Neither is a headline pair, and both are reported here as layerwise controls.

\subsection{Channel Ablation Sweep}
\label{app:c1details}

For each of the three normalization-scale tensors in Qwen3-VL-8B's language backbone (\texttt{input\_layernorm} $\gamma_{\mathrm{in}}$,
\texttt{q\_norm} $\gamma_q$, \texttt{k\_norm} $\gamma_k$), we identify the top-$K$ channels by $|\gamma|$ at each layer and replace those entries with the layer-wise channel mean.
This disables the amplifier while leaving the residual stream and projections untouched.
We fix $K=10$ for all three tensors, so the ablation targets the largest-magnitude amplifier channels at each layer.
Joint kill (all three norms simultaneously) costs 44.36 pt on IFEval, against a sum of individual costs of 19.96 pt, a 24.40-pt super-additive interaction.

\subsection{Random-\texorpdfstring{$\Delta W$}{Delta W} Control}
\label{app:e2details}

We add Gaussian noise $G_W \sim \mathcal{N}(0, \sigma_W^2 I)$ to each attention / MLP projection $W$ of Qwen2.5-7B-Instruct with $\sigma_W$ chosen so that the per-sub-module relative Frobenius norm $\frob{G_W}/\frob{W}$ exactly matches the corresponding ratio observed between Qwen2.5-VL-7B's extracted LM and Qwen2.5-7B-Instruct.
The $\gamma$ parameters of every RMSNorm are left untouched.
The result is a model whose per-submodule relative Frobenius distances from the reference LLM match those observed for Qwen2.5-VL, but whose perturbation direction is random.
IFEval under the instruct protocol drops to 10.72 ($\Delta=-61.37$ against the reference LLM at 72.09), compared with the observed Qwen2.5-VL gap of $-9.6$ pt.

\section{Sink Strength Predictor}
\label{app:predictor_supp}

\subsection{Per-Pair Sink Strength Predictions}
\label{app:predictor_detail}

Per-pair numerical detail behind Figure~\ref{fig:predictor} (Section~\ref{sec:predictor}): the reference-LLM Sink Strength $S$, the observed VLM--reference IFEval gap, and the leave-one-pair-out linear prediction.

\begin{table}[ht]
\centering
\setlength{\tabcolsep}{4pt}
\fontsize{8}{10}\selectfont
\begin{tabular}{@{}l c c c c@{}}
\toprule
Pair & QK & $S$ & Actual $\Delta$ & Held-out pred $\Delta$ \\
\midrule
Qwen3-VL    & per-head  & $+2.36$ & $-5.4$   & $-0.92$  \\
InternVL3.5 & per-head  & $+2.36$ & $-1.8$   & $-3.41$  \\
Qwen2.5-VL  & none      & $+1.18$ & $-9.6$   & $-10.78$ \\
InternVL3   & none      & $+1.18$ & $-8.9$   & $-10.95$ \\
LLaVA-OV    & none      & $+1.44$ & $-7.9$   & $-9.04$  \\
Molmo2-O    & layerwise & $+0.24$ & $-18.7$  & $-13.74$ \\
\bottomrule
\end{tabular}
\caption{Per-pair Sink Strength $S$ measured on the reference LLM, the observed IFEval gap between the VLM-LM and reference LLM, and the leave-one-pair-out linear prediction.}
\label{tab:predictor}
\end{table}

\subsection{IFEval-Specific Predictor (17 pairs)}
\label{app:predictor_17pair}

The IFEval-specific predictor panel (Table~\ref{tab:predictor_17pair}) measures Sink Strength $S$ against the VLM--reference IFEval prompt-strict $\Delta$ on an extended 17-pair set: the 6 pairs plus 11 additional VLMs that span MoE architectures (Qwen3-VL-30B-A3B-Instruct, MolmoE-1B), four reference-LLM families (Qwen, Olmo, Mistral, Llama-3.1) plus the MoE OLMoE, and backbone sizes $1.5$B--$32$B (full pair-to-reference-LLM mapping with Hugging Face URLs in Appendix~\ref{app:datasets}).
The panel therefore covers MoE and dense architectures, four reference-LLM families, and $1.5$B--$32$B scale.
Over it, Spearman $\rho(S, \Delta_{\mathrm{IFEval}}) = +0.88$.
Two caveats bound what a single rank correlation can carry here.
Pairs that share a reference LLM share $S$ by construction, so the $17$ pairs supply $12$ distinct $S$ values, and $13$ of the $17$ are Qwen-family, so the four non-Qwen backbones are single points rather than a family-level sample.
Two entries illustrate the resolution limit of a reference-LLM statistic.
Ovis2-8B shares $S=+1.18$ with Qwen2.5-VL and InternVL3 but loses $-17.9$ pt against
their $-9.6$ and $-8.9$, and Qwen3-VL-30B-A3B is the one pair whose VLM-LM exceeds its
reference ($+1.7$ pt).
$S$ orders backbones by risk; it does not predict the magnitude a particular VL recipe
will realize.

\begin{table*}[t]
\centering
\setlength{\tabcolsep}{6pt}
\fontsize{8}{10}\selectfont
\begin{tabular}{@{}l l l c r@{}}
\toprule
Pair & Reference LLM & QK class & $S$ & $\Delta$IFEval \\
\midrule
Qwen2.5-VL              & Qwen2.5-7B-Instruct       & none      & $+1.18$ & $-9.6$ \\
Qwen3-VL-8B             & Qwen3-8B                  & per-head  & $+2.36$ & $-5.4$ \\
InternVL3-8B            & Qwen2.5-7B-Instruct       & none      & $+1.18$ & $-8.9$ \\
InternVL3.5-8B          & Qwen3-8B                  & per-head  & $+2.36$ & $-1.8$ \\
LLaVA-OV-7B             & Qwen2-7B-Instruct         & none      & $+1.44$ & $-7.9$ \\
Molmo2-O-7B             & Olmo-3-7B-Instruct        & layerwise & $+0.24$ & $-18.7$ \\
\addlinespace[2pt]
Qwen3-VL-4B             & Qwen3-4B                  & per-head        & $+2.36$ & $-5.0$ \\
Qwen3-VL-2B             & Qwen3-1.7B                & per-head        & $+1.84$ & $-5.2$ \\
Ovis2-8B                & Qwen2.5-7B-Instruct       & none            & $+1.18$ & $-17.9$ \\
Ovis2-34B               & Qwen2.5-32B-Instruct      & none            & $+1.39$ & $-6.7$ \\
MolmoE-1B               & OLMoE-1B-7B-0924-Instruct & layerwise (MoE) & $+1.34$ & $-9.4$ \\
InternVL3-2B            & Qwen2.5-1.5B-Instruct     & none            & $+0.90$ & $-12.2$ \\
InternVL3-14B           & Qwen2.5-14B-Instruct      & none            & $+1.30$ & $-5.2$ \\
MiniCPM-V-4.5           & Qwen3-8B                  & per-head        & $+2.36$ & $-4.4$ \\
LLaVA-NeXT-Mistral-7B   & Mistral-7B-Instruct-v0.2  & none            & $-0.19$ & $-10.9$ \\
Idefics3-8B             & Llama-3.1-8B-Instruct     & none            & $-0.21$ & $-22.7$ \\
Qwen3-VL-30B-A3B        & Qwen3-30B-A3B-Instruct-2507 & per-head (MoE)  & $+1.66$ & $+1.7$ \\
\midrule
\multicolumn{3}{l}{$\rho_{\mathrm{Spearman}}(S, \Delta_{\mathrm{IFEval}})$ on $17$ pairs} & & $+0.88$ \\
\bottomrule
\end{tabular}
\caption{IFEval-specific 17-pair predictor panel. Sink Strength $S$
measured on the reference LLM versus the VLM-LM--reference-LLM
IFEval prompt-strict $\Delta$. The panel spans dense and MoE
architectures, all three QK-norm classes (per-head / layerwise / none), four reference-LLM families (Qwen, Olmo, Mistral, Llama-3.1) plus the MoE OLMoE, and backbone sizes $1.5$B--$32$B. Hugging Face URLs and paper references for every pair are in Appendix~\ref{app:datasets}.}
\label{tab:predictor_17pair}
\end{table*}

\subsection{Multi-Task Generalization (9 pairs)}
\label{app:predictor_extended}

Where Appendix~\ref{app:predictor_17pair} widens the model set at a fixed task, this panel widens the task set instead.
It adds three sub-10B Qwen-family pairs to the six (Qwen3-VL-4B, Qwen3-VL-2B, InternVL3-2B), which extends the panel down to 1.5B reference LLMs while keeping all four format-sensitive tasks on the unified protocol.
Table~\ref{tab:predictor_extended} reports per-pair Sink Strength $S$ and the four VLM--reference deltas.
The IFEval correlation is $+0.945$ against $+0.971$ on the six, and the other three tasks fall between $+0.82$ and $+0.88$.

\begin{table*}[t]
\centering
\setlength{\tabcolsep}{6pt}
\fontsize{8}{10}\selectfont
\begin{tabular}{@{}l l l c r r r r@{}}
\toprule
Pair & Reference LLM & QK & $S$ & $\Delta$IFEval & $\Delta$GSM8K & $\Delta$EQ-Bench & $\Delta$GPQA-D \\
\midrule
Qwen2.5-VL          & Qwen2.5-7B-Instruct      & none      & $+1.18$ & $\mathbf{-9.6}$  & $-14.3$ & $-11.2$ & $-9.8$ \\
Qwen3-VL-8B         & Qwen3-8B                 & per-head  & $+2.36$ & $\mathbf{-5.4}$  & $-2.1$  & $-2.4$  & $-1.9$ \\
InternVL3-8B        & Qwen2.5-7B-Instruct      & none      & $+1.18$ & $\mathbf{-8.9}$  & $-12.5$ & $-10.3$ & $-7.1$ \\
InternVL3.5-8B      & Qwen3-8B                 & per-head  & $+2.36$ & $\mathbf{-1.8}$  & $-1.5$  & $-0.8$  & $-0.6$ \\
LLaVA-OV-7B         & Qwen2-7B-Instruct        & none      & $+1.44$ & $\mathbf{-7.9}$  & $-11.9$ & $-8.6$  & $-8.4$ \\
Molmo2-O-7B         & Olmo-3-7B-Instruct       & layerwise & $+0.24$ & $\mathbf{-18.7}$ & $-42.7$ & $-64.6$ & $-19.7$ \\
\addlinespace[2pt]
Qwen3-VL-4B$^{\star}$   & Qwen3-4B               & per-head  & $+2.36$ & $\mathbf{-5.0}$  & $+13.9$ & $+1.4$  & $-2.0$ \\
Qwen3-VL-2B$^{\star}$   & Qwen3-1.7B             & per-head  & $+1.84$ & $\mathbf{-5.2}$  & $+27.5$ & $-0.7$  & $-6.1$ \\
InternVL3-2B$^{\star}$  & Qwen2.5-1.5B-Instruct  & none      & $+0.90$ & $\mathbf{-12.2}$ & $-18.0$ & $-9.0$  & $-6.1$ \\
\midrule
\multicolumn{4}{l}{$\rho_{\mathrm{Spearman}}(S, \Delta)$ on $9$ pairs} & $+0.945$ & $+0.877$ & $+0.860$ & $+0.821$ \\
\bottomrule
\end{tabular}
\caption{Extended 9-pair predictor panel. $\Delta$ denotes VLM-LM minus the corresponding reference LLM (negative $=$ lower performance relative to the reference) on IFEval prompt-strict, GSM8K-CoT flexible-extract, EQ-Bench v2.1, and GPQA-Diamond-CoT flexible-extract; all evaluated under the unified protocol (Appendix~\ref{sec:protocol}), except that the Molmo2-O GPQA baseline is the published Olmo-3-7B-Instruct score rather than our own matched run. $^{\star}$ marks the three pairs added beyond the six, which span $S \in [0.90, 2.36]$ and reference-LLM sizes $1.5$B--$4$B. Bold $\Delta$IFEval column is the headline metric.}
\label{tab:predictor_extended}
\end{table*}

\subsection{Sink Strength Design Ablations}
\label{app:sink_design_ablations}

Section~\ref{sec:predictor} defines $S$ as a median taken over the last $\sim\!10$ decoder layers.
This subsection isolates the three choices that definition makes.

\paragraph{Why the late layers.}
We split each reference LLM's decoder into equal early, middle, and late thirds and recompute $S$ within each band (Table~\ref{tab:sink_design}).
Sink concentration and its correlation with the outcome rise together with depth.
The early third carries almost no signal, while the late third alone reproduces the full predictor at $\rho = 0.97$.
The late third overlaps the last-$\sim\!10$-layer default for every headline pair, so the emphasis on late layers reflects where the sink lives rather than a tuned window.

\begin{table}[ht]
\centering
\setlength{\tabcolsep}{8pt}
\fontsize{8}{10}\selectfont
\begin{tabular}{@{}l c c@{}}
\toprule
Third & $S$ mean & $\rho(S, \Delta_{\mathrm{IFEval}})$ \\
\midrule
Early  & $+0.20$ & $+0.09$ \\
Middle & $+0.91$ & $+0.79$ \\
Late   & $+1.56$ & $+0.97$ \\
\bottomrule
\end{tabular}
\caption{Sink Strength recomputed within equal thirds of the decoder, on the six pairs. Both the magnitude of $S$ and its rank correlation with the VLM--reference IFEval gap increase with depth.}
\label{tab:sink_design}
\end{table}

\paragraph{Sensitivity to the window size.}
Recomputing $S$ over the last $K$ layers for $K \in \{3, 5, 10, 15, 20\}$ leaves the rank correlation unchanged at $\rho = 0.971$ for every $K$.
The magnitudes of $S$ shift with $K$, but no pairwise ordering reverses, so the last-$\sim\!10$-layer window is a robust choice rather than a tuned one.

\paragraph{Why the median.}
The statistic should report whether the \emph{typical} head forms a sink, which is what a median over $(\ell, h, q)$ measures, whereas a mean is inflated by a few strong heads.
Substituting the mean preserves both the sign and the pair ordering, so this choice changes none of the conclusions drawn from $S$.
Appendix~\ref{app:molmo} is the case that motivates it, where a strong model-level aggregate coexists with a weak per-head distribution that only the per-head statistic exposes.

\section{Controls}
\label{app:controls}

This appendix details the modality comparison of Section~\ref{sec:problem_setup} and the controls of Section~\ref{sec:negative}:
a text-only Qwen2.5-family endpoint (Appendix~\ref{app:lm_sft}),
a step-wise VL versus text-only trajectory from a shared 3B base
(Appendix~\ref{app:vl_tulu}), a post-pretraining $\qknorm$ injection experiment (Appendix~\ref{app:c3details}), and a weight-merging sweep (Appendix~\ref{app:method_sweep}).

\subsection{Modality Comparison: Text-Only Endpoint}
\label{app:lm_sft}

The Section~\ref{sec:problem_setup} modality comparison (Table~\ref{tab:lm_sft_control}) uses Qwen2.5-7B-Instruct-1M, the Qwen team's official long-context text-only release, as a Qwen2.5-family further-training endpoint (Hugging Face URL in Appendix~\ref{app:datasets}).

Both this model and Qwen2.5-VL-7B-Instruct are evaluated against the same Qwen2.5-7B-Instruct text-only reference.
The text-only variant is evaluated under the same 0-shot instruct protocol as the rest of the paper (Appendix~\ref{sec:protocol}), with batch size 8 on a single A6000.

\begin{figure*}[!t]
\centering
\includegraphics[width=\textwidth]{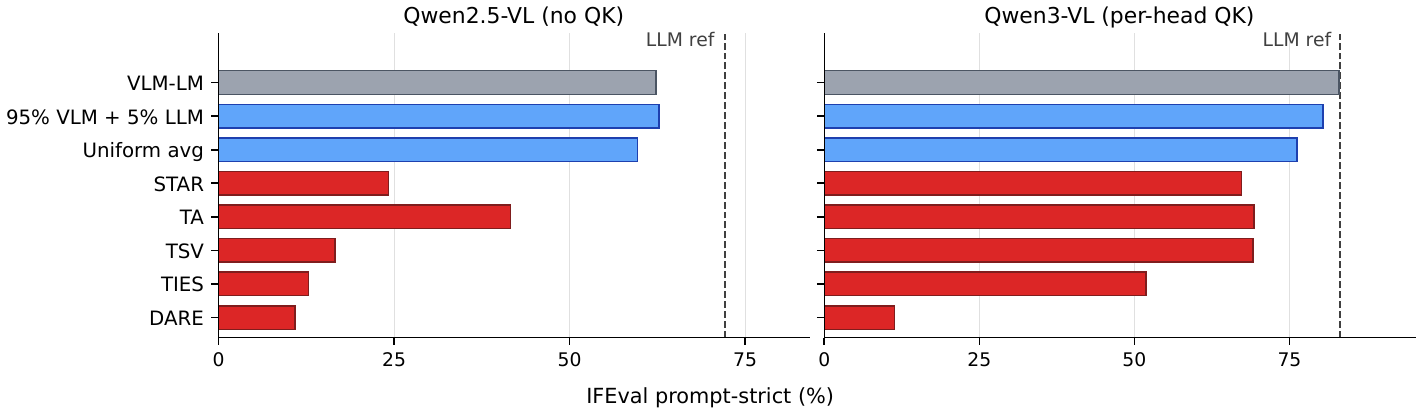}
\caption{IFEval prompt-strict for Qwen2.5-VL (non-protected) and Qwen3-VL (per-head $\qknorm$) under seven weight-merging settings. Dashed line is each pair's reference LLM.}
\label{fig:method_sweep}
\end{figure*}

\subsection{Modality Control: VL vs.\ Text-Only Trajectory}
\label{app:vl_tulu}

The Appendix~\ref{app:lm_sft} comparison contrasts released $7$B endpoints against a common text-only reference.
This subsection instead tracks the two modalities step by step on the $3$B base of the injection experiment, so the sink can be observed while it changes.
Starting from \texttt{Qwen2.5-3B-Instruct}, we run vision fine-tuning on image-text instruction data (the Stage-2 recipe of Appendix~\ref{app:c3details}) and a matched text-only run on T\"ulu general instruction SFT~\citep{lambert2024tulu}, at the same effective batch size, learning rate, and number of steps, matching the optimization settings while varying the training modality and data.
At each checkpoint we measure Sink Strength $S$ on the trained backbone, vision understanding on SEEDBench-IMG~\citep{li2023seedbench}, and IFEval prompt-strict.

Table~\ref{tab:vl_tulu} reports the trajectories.
The vision run gains vision understanding while its sink collapses, with SEEDBench-IMG rising from $60.1$ to $64.8$ and $S$ falling from $+0.41$ at the shared base to $+0.20$ at step $10$k, and IFEval sitting in the $34$--$44$ band throughout.
The text-only run moves the sink in the opposite direction, holding $S$ between $+0.58$ and $+0.65$, and gives up about $5$ pt of IFEval.
The step-$0$ IFEval of $59.9$ reproduces the published \texttt{Qwen2.5-3B-Instruct} score of $58.2$~\citep{qwen2025qwen25technicalreport} to within $1.7$ pt, so the drops start from a correctly measured baseline.
Within the vision run the two curves are not step-by-step aligned.
IFEval takes most of its loss by step $1$k, before $S$ has fallen, and then recovers slightly while $S$ continues to decline, so $S$ tracks the regime the run ends in rather than the step-level fluctuation.

\begin{table}[ht]
\centering
\setlength{\tabcolsep}{4pt}
\fontsize{8}{10}\selectfont
\begin{tabular}{@{}r c c c c c@{}}
\toprule
 & \multicolumn{3}{c}{VL adaptation} & \multicolumn{2}{c}{T\"ulu (text-only)} \\
\cmidrule(lr){2-4} \cmidrule(lr){5-6}
Step & $S$ & SEED & IFEval & $S$ & IFEval \\
\midrule
$0$      & $+0.41$ & --     & $59.9$ & $+0.41$ & $59.9$ \\
$1$k     & $+0.54$ & $60.1$ & $36.8$ & $+0.65$ & $53.1$ \\
$2$k     & $+0.49$ & $60.9$ & $37.2$ & $+0.58$ & $53.6$ \\
$3$k     & $+0.35$ & $62.1$ & $34.2$ & $+0.61$ & $55.3$ \\
$4$k     & $+0.30$ & $63.0$ & $40.5$ & $+0.60$ & $56.9$ \\
$5$k     & $+0.24$ & $63.9$ & $39.7$ & $+0.58$ & $54.3$ \\
$6$k     & $+0.22$ & $63.9$ & $38.6$ & $+0.58$ & $55.1$ \\
$7$k     & $+0.19$ & $64.3$ & $42.3$ & $+0.58$ & $55.3$ \\
$8$k     & $+0.24$ & $64.7$ & $40.7$ & $+0.58$ & $53.8$ \\
$9$k     & $+0.20$ & $65.0$ & $43.1$ & $+0.58$ & $54.9$ \\
$10$k    & $+0.20$ & $64.8$ & $43.6$ & $+0.58$ & $55.8$ \\
\bottomrule
\end{tabular}
\caption{Step-wise modality control on \texttt{Qwen2.5-3B-Instruct}. $S$ is Sink Strength on the trained backbone, SEED is SEEDBench-IMG overall accuracy, and IFEval is prompt-strict. Both runs start from the same base and use the same batch size, learning rate, and step count, while differing in training modality and data.}
\label{tab:vl_tulu}
\end{table}

\subsection{Injection Experiment Training Recipe}
\label{app:c3details}

The injection experiment trains two 3B VLMs from the same Qwen2.5-3B-Instruct base under the same LLaVA Stage-1 / Stage-2 recipe.
The two variants are \emph{vanilla} (no $\qknorm$) and \emph{$\qknorm$ injected} with unit-scale initialization $\gamma_q = \gamma_k = 1$.
Vision tower and projector come from Qwen2.5-VL-3B-Instruct.
Both runs proceed on dedicated $2 \times$~NVIDIA H100 nodes with DeepSpeed ZeRO-2.
Total walltime is $\sim 47$~h per variant ($\sim 7$~h Stage~1, $\sim 40$~h Stage~2).

\begin{table}[ht]
\centering
\setlength{\tabcolsep}{3pt}
\fontsize{8}{10}\selectfont
\begin{tabular}{@{}l c c@{}}
\toprule
Hyperparameter & Stage 1 (align) & Stage 2 (instruct) \\
\midrule
Trainable params       & projector only            & full LM + projector \\
Dataset                & LLaVA-Pretrain ($558$K)   & LLaVA-Mix ($665$K) \\
Image resolution       & $336^2$ (112{,}896 px)    & $672^2$ (451{,}584 px) \\
Per-device batch       & $32$                      & $4$ \\
Grad accumulation      & $2$ (NGPU$=2$)            & $16$ \\
Effective batch        & $128$                     & $128$ \\
Optimizer              & AdamW                     & AdamW \\
Learning rate          & $1\!\times\!10^{-4}$      & $2\!\times\!10^{-5}$ \\
LR schedule            & const $+$ warmup          & cosine \\
Warmup ratio           & $0.03$                    & $0.03$ \\
Weight decay           & $0.0$                     & $0.0$ \\
Max grad norm          & $1.0$                     & $1.0$ \\
Epochs                 & $1$                       & $1$ \\
Steps                  & $4361$                    & $\sim 5195$ \\
Precision              & bf16                      & bf16 \\
DeepSpeed              & ZeRO-2                    & ZeRO-2 \\
Save every             & end of stage              & $500$ steps \\
Save total limit       & ---                       & $2$ \\
\bottomrule
\end{tabular}
\caption{Training recipe for the post-pretraining $\qknorm$ injection experiment (vanilla and injected variants are matched on all rows; the $\qknorm$ module is the only difference).}
\label{tab:c3_recipe}
\end{table}

\begin{figure}[ht]
\centering
\includegraphics[width=\columnwidth]{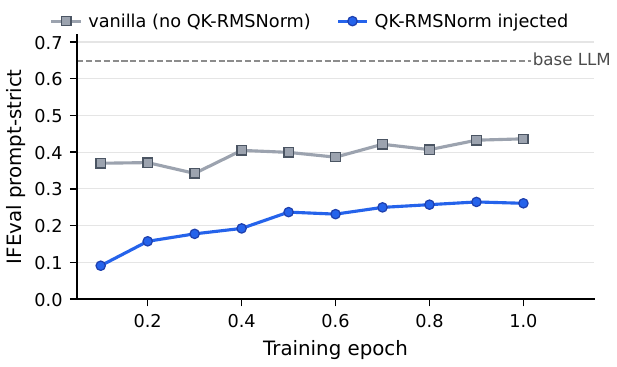}
\caption{IFEval prompt-strict over the LLaVA Stage-2 training epoch for the controlled post-pretraining $\qknorm$ injection pair (vanilla vs.\ injected, matched in all other settings).}
\label{fig:c3_progression}
\end{figure}

\subsection{Weight-Merging Sweep on Format-Sensitive Benches}
\label{app:method_sweep}

The body reports IFEval on the two pairs (Qwen2.5-VL, Qwen3-VL) for seven weight-merging settings: the asymmetric $95\%$ VLM $+$ $5\%$ LLM mix, uniform averaging, full Task Arithmetic ($\lambda{=}1$), Task Singular Vectors ($\alpha{=}1$), STAR ($\eta{=}40\%$), TIES (density $0.2$, $\lambda{=}1$), and DARE (mask rate $0.9$, $\lambda{=}1$).
We also ran the same settings on EQ-Bench, GSM8K-CoT flexible-extract, and GPQA-Diamond-CoT flexible-extract, and summarize the pattern here rather than reproducing three further figures.
It matches IFEval, in that the mild mixes (95/5, uniform averaging) avoid substantial additional degradation, full-strength operators (TA, TSV, STAR) drop the non-protected pair by tens of points, and sparsified merges (TIES, DARE) are the most damaging across both pairs.

\end{document}